\documentclass{article}
\usepackage{iclr2027_conference}

\usepackage{pratik}
\usepackage{fvextra}
\graphicspath{{./figures/}}

\title{
\noindent the dichotomy between pattern recognition and step-by-step reasoning
}

\author{Amrut Nadgir, Vijay Balasubramanian, Pratik Chaudhari\\
University of Pennsylvania\\
\texttt{\{amrutn,vijay,pratikac\}@upenn.edu}
}

\iclrfinalcopy 

\begin{document}
\clearpage

\maketitle
\lhead{}

\begin{abstract}
We argue that pattern recognition and step-by-step reasoning are two ends of a spectrum. A large language model (LLM) learns to reason step-by-step when data is structured such that the next token depends on a small amount of preceding context. Inference in LLMs resembles pattern recognition when the next token depends on a large amount of preceding context. If the next token depends on only the $c$ most recent tokens, reasoning traces are paths on a De Bruijn graph whose nodes are $c$-length contexts and edges are next-token transitions between contexts. The set of reasoning traces of a task forms a directed acyclic subgraph of the De Bruijn graph. An LLM that has learned all edges of this subgraph can compose them to solve longer, unseen tasks, i.e., it reasons step-by-step. We prove that the number of edges is vanishingly small compared to the number of reasoning traces. Empirically, the number of training samples a transformer needs is a power law in the number of edges, so learning to reason step-by-step is sample efficient. We can induce De Bruijn structure in any task by maintaining a ``state'' that makes future reasoning independent of the past. The frequency of states in the reasoning trace determines $c$. We show, by fine-tuning Qwen2.5-1.5B-Instruct to solve equations and answer questions about stories, that frequent states (small $c$) result in higher accuracy but greater fragility to perturbations at test time. LLMs trained with a large $c$ are only as good as models that perform pattern recognition without reasoning. A moderate density of states balances accuracy and robustness. We show that real-world data has De Bruijn structure: Qwen3-14B and Qwen3-32B retain over 75\% of their accuracy on GSM8K, MATH-500 and GPQA-Diamond when attention is restricted to a sliding window less than 15\% as long as the full reasoning trace.

\end{abstract}

\section{introduction}

Large language models (LLMs) are successful at solving problems thought to require precise, step-by-step reasoning~\citep{guoDeepSeekR1IncentivizesReasoning2025a,glazerFrontierMathBenchmarkEvaluating2025,hubertOlympiadlevelFormalMathematical2026,lewkowyczSolvingQuantitativeReasoning2022,jimenezSWEbenchCanLanguage2023}. This is because they can decompose a difficult prediction into a sequence of steps~\citep{weiChainofThoughtPromptingElicits2022,nadgirHowDoesChain2026b}. This inference strategy seems inconsistent with how LLMs are trained: to predict the next token by imitating statistical patterns in the data. They should use ``pattern recognition'' to map ``similar'' inputs to ``similar'' outputs and thereby generalize to unseen data~\citep{sejnowskiDeepLearningRevolution2018}. While a large body of machine learning literature characterizes how generalization is achieved in pattern recognition, there is relatively little literature on why reasoning emerges in LLMs and how it generalizes. It has been argued that reasoning in LLMs is a fundamentally different inference paradigm \citep{prystawskiWhyThinkStep2023a}.

In this paper, we argue that step-by-step reasoning and pattern recognition are not mutually exclusive, they are two ends of a spectrum. When the next token depends on only a small amount of preceding context, LLMs learn to reason step-by-step. Inference in LLMs resembles pattern recognition when the training samples do not have this property.

Our study is motivated by two recent works. \citet{nadgirHowDoesChain2026b} showed that chain-of-thought is successful because it enables the LLM to decompose the overall task into a sequence of smaller classification problems where the classes correspond to the tokens being predicted at each step. In effect, an LLM uses the prompt to traverse a prefix-tree of the set of its reasoning traces. The length of the reasoning trace, say $n$, determines the depth of this prefix-tree. \citet{cagnettaDerivingNeuralScaling2026c} showed that LLMs trained on $D$ tokens of natural language text can reliably capture correlations among only $\sim \sqrt{D}$ nearby tokens. This puts a limit on the length of reasoning traces. For instance, one would need $D \sim 10^{10}$ tokens of in-distribution samples (directly relevant to the task) in order to learn reasoning traces that are of length $n \sim 10^5$. Pre-training datasets are big, $\sim 10^{13}$ tokens, but only a small fraction consists of reasoning traces relevant to a task. So there must be some additional structure that enables LLMs to generalize to newer and longer tasks.

\begin{wrapfigure}[22]{L}{0.3\linewidth}
    \centering
    \includegraphics[width=\linewidth,
    trim={0.25cm 0.5cm 0.7cm 1.1cm}, clip]{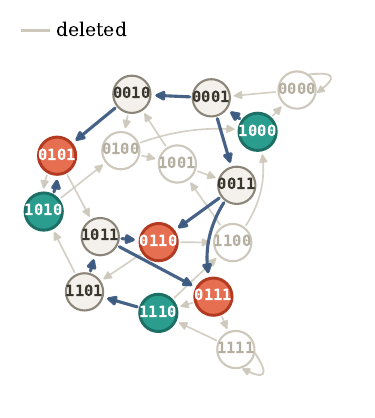}
    \caption{\textbf{De Bruijn graph:} Edges (both bold and gray) show transitions from one 4-gram to another by appending a token. Bold edges form a directed acyclic subgraph (DAG). Green and orange shaded nodes are start and answer nodes in the DAG respectively. Each path corresponds to a reasoning trace.}
    \label{fig:debruijn}
\end{wrapfigure}
Our key technical argument is that, if the set of valid next tokens can be predicted using a limited preceding context, say of length $c$, then we can compress the prefix-tree into a data structure that can be learned with few samples. The LLM transitions from one $c$-length context to another by producing a token. A De Bruijn graph $B(V, c)$ with vocabulary $V$ and context $c$ has one node for each unique $c$-length context and one edge for each valid transition between contexts. \cref{fig:debruijn} gives an example of such a graph with $V =2$ and $c = 4$. Any finite-length prefix-tree with this structure is a directed acyclic subgraph of the De Bruijn graph, henceforth called a ``De Bruijn DAG''. This is because including only forward-pointing edges ensures that all reasoning paths are finite.

\textbf{1. Learning to reason step-by-step is sample efficient when training data has De Bruijn structure.}
We prove that the number of paths required to cover every edge in a De Bruijn DAG is a vanishingly small fraction of the total number of paths. Therefore, if reasoning has De Bruijn structure, then transitions between successive contexts can be learnt with few samples. We also prove that sampling short reasoning traces is sufficient to cover every edge in the DAG. This makes it possible to generalize to longer reasoning problems than those in the training data. We show experimentally that the number of training samples required by a transformer to reproduce a set of reasoning traces scales as a power law of the number of edges in its De Bruijn DAG. Therefore, learning reasoning traces with the De Bruijn structure is indeed sample efficient. In addition, if the data has De Bruijn structure, a transformer trained only on short traces generalizes to much longer ones. The gap between achievable reasoning length at test time and available reasoning length at training time grows with the size of the De Bruijn DAG.

\textbf{2. The degree of De Bruijn structure in the training data determines how much an LLM uses pattern recognition or step-by-step reasoning.} We can induce De Bruijn structure on any reasoning task by training an LLM to maintain a ``state'' that makes the future reasoning trace independent of the past. We fine-tune Qwen2.5-1.5B-Instruct to solve equations using reasoning traces that contain the current state at periodic intervals. More frequent state emissions entails that each succeeding token depends on fewer past tokens (i.e., more De Bruijn structure). We show that LLMs that are trained on data which contains the state more frequently have higher accuracy with fewer samples. They also perform reasoning more faithfully, i.e., when they produce the correct answer, the sequence of reasoning steps is also correct. However, such models are comparatively brittle to answer-preserving perturbations in their chain-of-thought at test-time. They tend to propagate errors in their predicted state into future reasoning steps. LLMs trained on reasoning traces without De Bruijn structure are no more accurate than pattern recognition, i.e., using no reasoning at all. This suggests that an LLM should be trained to use reasoning with a moderate amount of De Bruijn structure to balance accuracy on the task against robustness to unforeseen inputs. Our results generalize to a natural language task where the LLM answers questions about a story.

\textbf{3. Real-world reasoning traces exhibit De Bruijn structure.}
We apply an attention mask to Qwen3-14B and Qwen3-32B so that they can only attend to the prompt and $c$ most recent tokens. Even for a value of $c$ that is below one-tenth of the mean length of the reasoning trace, both models retain at least 90\% of full-context accuracy on GSM8K \citep{gsm8k}, 80\% on MATH-500 \citep{hendrycks2021measuringmathematicalproblemsolving} and 65\% on GPQA-Diamond \citep{reinGPQAGraduateLevelGoogleProof2023}. This, together with the previous two results, suggests that the reasoning traces used to solve these benchmark tasks exhibit some degree of De Bruijn structure, even if the state is not explicitly represented in the traces.

%
\section{learning to reason step-by-step is sample efficient when training data has de bruijn structure}

Suppose that the set of valid next tokens can be predicted using only the most recent $c$ tokens. In this section, we characterize the sample complexity of learning step-by-step reasoning on such data that has De Bruijn structure.

We can represent any set of reasoning traces as a prefix-tree. Any finite-length prefix-tree with De Bruijn structure is a directed acyclic subgraph of the De Bruijn graph $B(V, c)$, called a De Bruijn DAG, or just DAG in short. We can construct an example DAG as follows. Pick a topological ordering $\pi$ of the nodes of $B(V, c)$, mark the last $N \geq 1$ nodes as ``answer'' nodes and the first $S \geq 1$ nodes as the ``start'' nodes. The former correspond to sequences where reasoning traces terminate while the latter correspond to initial sequences. Prune all outgoing edges from answer nodes, all incoming edges into start nodes, all edges $(u,v)$ with $\pi(v) < \pi(u)$ in the topological ordering, and all edges and nodes that do not lie on start-to-answer reasoning traces. Let us call the DAG created by this procedure $D_\pi$. In~\cref{app:debruijn_math}, we show that $D_{\pi}$ can have the maximal number of distinct start-to-answer paths.

We will next characterize the number of samples required to learn all the edges of the DAG. Let $P_{\text{min}}$ be the minimum number of start-to-answer paths required to cover every edge in the DAG and let $P_{\text{tot}}$ be the total number of paths in the DAG. Let $L_{\text{cover}}$ be the minimum length of start-to-answer paths required to cover every edge and let $L_{\text{max}}$ be the length of the longest path in the DAG. We prove the following two theorems in~\cref{app:debruijn_math}.

\begin{theorem}
\label{theorem1}
Let $S,N < \binom{c+m-1}{m-1}$ for some integer $1 < m < V/2$, and let $|E|$ be the number of edges in $D_\pi$. There exists an ordering $\pi$ such that
\beq{
   P_{\text{tot}} \geq \sum_{t=0}^{V-2m} \binom{V-2m}{t}^c
   \quad \text{ and } \quad
   P_{\text{min}} \leq \abs{E} \leq \f{V^c(V+1)}{2}.
\label{eq:ptol_pmin}
}
The ratio $P_{\text{min}}/P_{\text{tot}} \leq \abs{E}/P_{\text{tot}}\to 0$ in the limit $V \to \infty$. Similarly, $P_{\text{min}}/P_{\text{tot}}\leq \abs{E}/P_{\text{tot}} \to 0$ in the limit $c\to\infty$ provided $V$ is sufficiently large.
\end{theorem}
\begin{theorem}
\label{theorem2}
Let $S,N < \binom{c+m-1}{m-1}$ for some integer $1 < m < V/2$. There exists an ordering $\pi$ such that
\beq{
   L_{\text{max}} \geq c(V-2m+1)\quad \text{ and }\quad
   L_{\text{cover}} \leq 2c+1.
\label{eq:lmax_lcover}
}
In the limit $V\to \infty$, the ratio $L_{\text{cover}}/L_{\text{max}} \to 0$.
\end{theorem}

\begin{figure}[!b]
\centering
\includegraphics[width=0.4\linewidth]{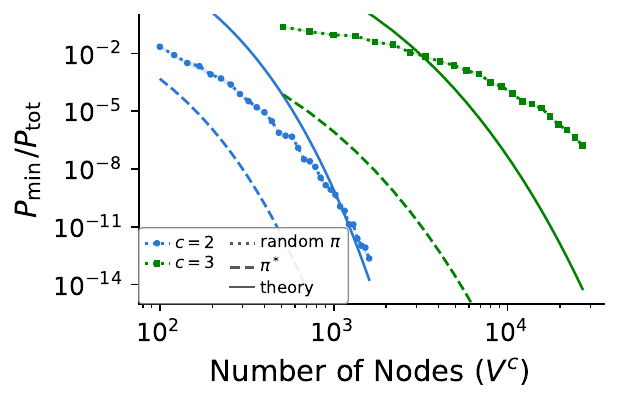}
\includegraphics[width=0.4\linewidth]{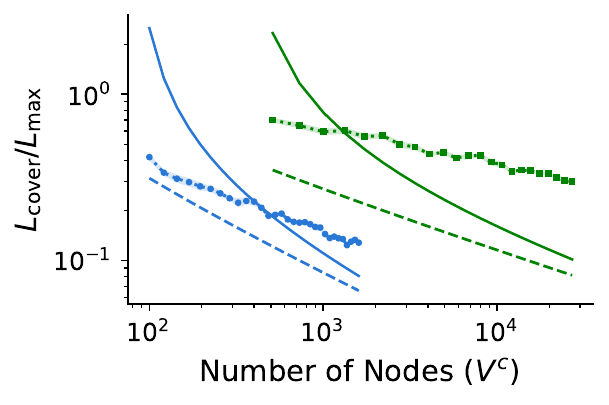}
\caption{\textbf{Data with De Bruijn structure can be learned with few, short traces}. We created a synthetic dataset of De Bruijn graphs and DAGs for different values of $c$ and $V$. 
\textbf{Left:} The minimum number of start-to-answer paths required to cover every edge in the DAG $P_{\text{min}}$ shrinks relative to the total number of paths $P_{\text{tot}}$ as the number of nodes grows.
\textbf{Right:} The minimum length of paths to cover every edge $L_{\text{cover}}$ shrinks relative to the longest path $L_{\text{max}}$.
\textbf{Details:} DAGs were created with $S=5$ start nodes, $N=10$ answer nodes. Dotted curves denote an average over random orderings for nodes in  the DAG. Solid curves correspond to using a particular ordering $\pi^*$ in the proof of \cref{theorem1,theorem2}. We computed $P_{\text{min}}$ using a min-flow problem and $L_{\text{cover}}$ as the maximum over edges for the length of the shortest path from any start node to any answer node containing that edge. For random $\pi$, curves show the mean ratio over random orderings and its standard error over ten DAGs (too small to see).
}
\label{fig:theory}
\end{figure}

Let us discuss some intuition behind these two results. The degree of the nodes in the De Bruijn graph $B(V, c)$ is $V$. The degree of the De Bruijn DAG is less than $V$ because of the pruned edges. Since $L_{\text{max}} \geq c (V - 2m + 1)$ in \cref{theorem2}, the total number of paths in the DAG is $V^{c (V - 2m + 1)}$ minus the number of pruned paths. If we ignore the pruned paths (for some intuition), the total number of paths is exponential in $V$. Since $L_{\text{cover}} \leq 2c + 1$, the minimum number of paths to cover all edges in the DAG is at most $V^{2c + 1}$, polynomial in $V$. This is why the ratio $P_{\text{min}}/P_{\text{tot}} \to 0$ as $V \to \infty$. In addition, for a large enough value of $V > 2m + 1$, we should expect the ratio $P_{\text{min}}/P_{\text{tot}} \to 0$ as $c \to \infty$. The proof of \cref{theorem1,theorem2} makes this intuition rigorous. \cref{theorem2} shows that it is sufficient for the training data to contain rather short reasoning traces. Altogether, these two results show that if reasoning traces have a De Bruijn structure, then an LLM with a large vocabulary can learn to reason step-by-step with short and few reasoning traces. These theorems are verified numerically in \cref{fig:theory}.

\begin{figure}[!htpb]
\centering
\includegraphics[width=0.4\linewidth]{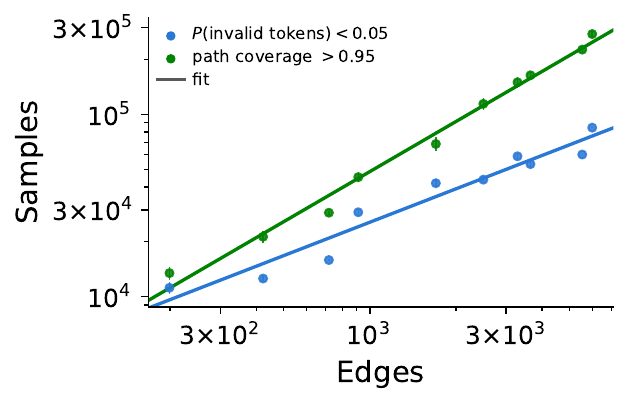}
\includegraphics[width=0.4\linewidth]{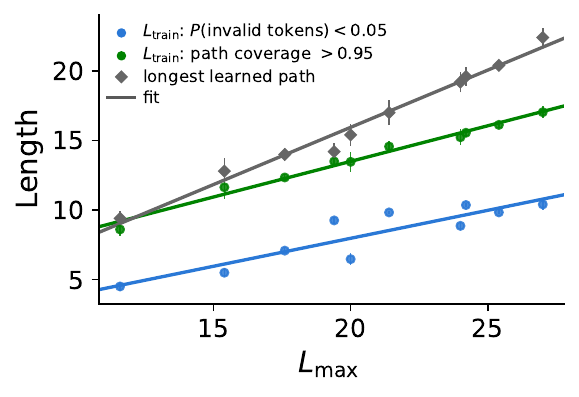}
\caption{\textbf{Transformers can learn paths on a De Bruijn DAG with few and short samples}.
We created a synthetic dataset of De Bruijn graphs and DAGs for different values of $c$ and $V$, and trained transformers on start-to-answer paths from 5 replicate DAGs for each configuration (random $\pi$, error bars are standard error). The green curve indicates the sample complexity (left) or the length of reasoning traces in the training data (right) required for achieving 95\% path coverage, i.e., when the trained transformer assigns probability greater than 10\% of the probability under the uniform distribution $1/P_{\text{tot}}$ to $>$95\% of paths. The blue curve corresponds to the situation where the trained transformer assigns a probability smaller than 5\% to invalid next-tokens that do not lie on valid reasoning paths in the DAG.
\textbf{Left:} The number of training samples required to reproduce paths of the DAG scales as a power law in the number of edges ($\sim x^{0.61}$ for the blue curve and $\sim x^{0.91}$ for the green curve).
\textbf{Right:} The length of reasoning traces in the training data scales  linearly in $L_{\text{max}}$ and is smaller than $L_{\text{max}}$. The gap between $L_{\text{train}}$ and the longest covered path grows with DAG size. Details in~\cref{app:synth_exp}.
}
\label{fig:synthetic_exp}
\end{figure}

\begin{remark}[\textbf{What is the ideal vocabulary size?}]
Since both ratios $P_{\text{min}}/P_{\text{tot}}$ and $L_{\text{cover}}/L_{\text{max}}$ vanish as $V \to \infty$, it may seem that large vocabulary sizes are ideal for learning step-by-step reasoning. This is counter to the empirical preference for intermediate vocabulary sizes that are calibrated to model size and the amount of training data \citep{taoScalingLawsVocabulary2024}. This discrepancy is resolved by considering the complexity of a next-token prediction. The out-degree of any node in $D_\pi$ increases with vocabulary size $V$, forcing the LLM to distinguish between a large number of valid next tokens at any step~\citep{nadgirHowDoesChain2026b}. The optimal vocabulary size is therefore governed by a trade-off between maximizing the sample efficiency of reconstructing the reasoning paths (favoring a large $V$) and minimizing errors when selecting the correct path (favoring a small $V$).
\end{remark}

In~\cref{fig:synthetic_exp} (left), we show using transformers trained on synthetic data, that the number of samples required to learn the paths in the DAG is a power law in the number of edges in the DAG. \cref{theorem1} shows that the number of edges $\abs{E}$ in the DAG scales significantly slower than the total number of paths $P_{\text{tot}}$. Together, these two results suggest that learning reasoning paths with De Bruijn structure becomes more sample efficient as the size of the DAG increases. In addition,~\cref{fig:synthetic_exp} (right) shows that models assign non-trivial probability to paths that are significantly longer than anything seen in the training data. This is consistent with the result in~\cref{theorem2}.
\begin{wrapfigure}[44]{L}{0.4\linewidth}
\includegraphics[width=\linewidth]{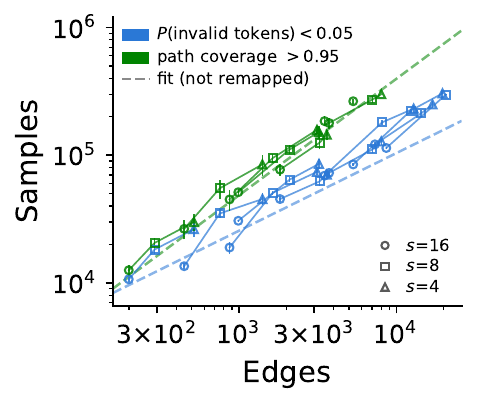}
\caption{\textbf{Sample complexity of learning remapped reasoning traces}.
Connected line segments correspond to remapping with block lengths $s = 4, 8, 16$; there are 10 different DAGs drawn on this plot. Dashed lines indicate the fit obtained in \cref{fig:synthetic_exp} for the original reasoning traces. Smaller values of $s$ have a larger number of edges and sample complexity. Learning on remapped traces follows a similar power-law trend as seen in \cref{fig:synthetic_exp}.}
\label{fig:remapped}
\begin{framed}
\footnotesize
\noindent \textbf{Task:} solve for x: -4*(2*x + 13) + 10 = -122. output your answer after \#\#\#\#.\\

\noindent \textbf{No state emission ($k = \infty$):} subtract 10 from both sides. divide both sides by -4. subtract 13 from both sides. divide both sides by 2. \#\#\#\# 10.\\

\noindent \textbf{Frequent state emission ($k=1$):} subtract 10 from both sides. current equation is -4*(2*x + 13) = -132. \textcolor{red}{add 5 to both sides}. current equation is -4*(2*x + 13) + 5 = -127. subtract 5 from both sides. current equation is -4*(2*x + 13) = -132. divide both sides by -4. current equation is 2*x + 13 = 33. subtract 13 from both sides. current equation is 2*x = 20. divide both sides by 2. current equation is x = 10. \#\#\#\# 10.
\end{framed}
\caption{\textbf{An equation-solving task.} The red text is a perturbation inserted at inference-time. The step-by-step reasoning model ($k=1$) adapted to this perturbation by undoing the operation in the next step.}
\label{fig:eqtask}
\end{wrapfigure}

\textbf{De Bruijn structure must be explicitly represented in the data to achieve low sample complexity.}
Some reasoning traces might not contain explicit De Bruijn structure but have some other hidden structure instead. A transformer may be able to learn these traces with few samples if it learns the hidden structure instead of the explicit next-token transitions. We design an experiment in \cref{fig:remapped} to isolate this phenomenon. Consider the set of start-to-answer paths on a De Bruijn DAG. Partition each path into consecutive blocks of $s$ tokens and, within each block, permute the elements of the vocabulary by a fixed, block-specific permutation. For example, sequences 1324 and 4344 with vocabulary size $V=4$ are remapped to 1321 and 4311 respectively if the block size $s = 2$ via the remapping $4 \to 1$, $2 \to 2$ in the second block. All reasoning traces in the dataset are permuted by the same set of block-specific permutations. This remapping is a bijection, so it preserves the underlying structure of the DAG. However, the structure is no longer explicit in the tokens because the permutations depend on position (even if the permutations are consistent across all traces).

Before remapping, a De Bruijn DAG is a highly compressed representation of the prefix-tree of reasoning paths. This is no longer true after remapping; each edge in the original DAG needs to be repeated in the remapped DAG if it appears in different blocks. The new DAG therefore has many more edges than the original one. If the transformer learns the underlying De Bruijn structure from the remapped sequences, it should only require as many samples as the original problem. If, instead, the transformer relies on the structure explicit in the trace, then sample complexity should scale with the edges in the remapped DAG. In \cref{fig:remapped}, we see the latter situation. This indicates that the number of samples a transformer needs to learn reasoning traces is small if there is explicit De Bruijn structure in the data; implicit structure may not lead to improvements in sample complexity. This is consistent with \cref{theorem1,theorem2} as well as the empirical observations of~\citet{vafa2024evaluating}.

\section{the degree of de bruijn structure in the training data determines how much an LLM uses pattern recognition or step-by-step reasoning}

We can induce De Bruijn structure in the reasoning traces for any task by training an LLM to maintain a ``state'' that makes future reasoning conditionally independent of the past. When states are frequently interspersed throughout the reasoning trace, reasoning has De Bruijn structure with a small context window $c$. Infrequently observed states correspond to a larger $c$, i.e., less De Bruijn structure. We will discuss two types of reasoning tasks where we train LLMs on data with different frequencies of state emission.

\paragraph{Equation-solving task}
We fine-tuned Qwen2.5-1.5B-Instruct~\citep{qwenQwen25TechnicalReport2025} to solve single-variable equations like the one in \cref{fig:eqtask}. A baseline model was fine-tuned to solve equations without reasoning. Other models were trained to emit a state once every $k\in \{1, \dots, 7, \infty\}$ ``operations'' in their reasoning. The $k = \infty$ case corresponds to no state emissions. Operations manipulate the initial equation towards a simplified state and are written explicitly in the reasoning trace, e.g., ``add 5 to both sides''. A state is the new equation after applying a sequence of operations to the previous state. We measured the accuracy of the final answer and the correctness of the intermediate reasoning for models trained with different state-emission intervals $k$. We also tested the robustness of trained models to perturbations in their reasoning. With probability $p\in \{0.05,0.1, 0.2,0.3\}$, we inserted a randomly chosen operation in place of the model's next predicted operation during decoding. The model then continued decoding with this random operation in its reasoning trace. Such perturbations force a step-by-step reasoning model to predict a different sequence of states but they do not change the final answer. A step-by-step reasoner must use extra operations to undo the random operation. See \cref{fig:eqtask} for an example.

\begin{wrapfigure}[10]{L}{0.45\linewidth}
\centering
\Large
\begin{adjustbox}{width=\linewidth}
\renewcommand{\arraystretch}{1.2}
\begin{tabular}{lrrr}
\toprule
State emission & Perturbation & Total number of & Fraction correct w/\\
interval & probability & correct answers & perturbations\\
\midrule
$k=1$ & 0.2 & 333 & 26.4\% \\
$k=1$ & 0.3 & 190 & 34.2\% \\
$k=\infty$ & 0.2 & 473 & 65.3\% \\
$k=\infty$ & 0.3 & 428 & 79.4\% \\
\bottomrule
\end{tabular}
\end{adjustbox}
\caption{The proportion of problems where the model predicted the correct answer despite perturbed reasoning.}
\label{tab:correct_prop}
\end{wrapfigure}
As \cref{fig:math_eval_acc} shows, models which emit the state after every step ($k=1$) can predict the answer more accurately with fewer training samples compared to those which emit the state less frequently. However, the $k=1$ models are less accurate than the $k=\infty$ models when perturbations are introduced, so frequent state-emissions cause models to be less robust to perturbations. The results in \cref{tab:correct_prop} support this conclusion: in the majority of cases where the $k=1$ model predicts the correct answer, there were no perturbations applied. To understand why models trained with $k = \infty$ are more robust than ones trained with $k = 1$, we can check the correctness of their reasoning traces. \cref{fig:math_eval_trends} (left) shows that across all perturbation probabilities, smaller emission intervals result in more correct reasoning steps. The discrepancy between trends in correctness of the reasoning and the accuracy of the final answer means that models with larger state-emission intervals can predict correct final answers despite incorrect reasoning. This finding can be explained by our theoretical arguments, which state that models that reason with less De Bruijn structure (large $c$) need to use statistical pattern recognition across their global context to predict the next token instead of step-by-step reasoning. This is because they do not densely sample all the edges of the DAG in their training data. There are three independent observations that support our explanation above.

\begin{figure}[!htpb]
\centering
\includegraphics[width=0.4\linewidth]{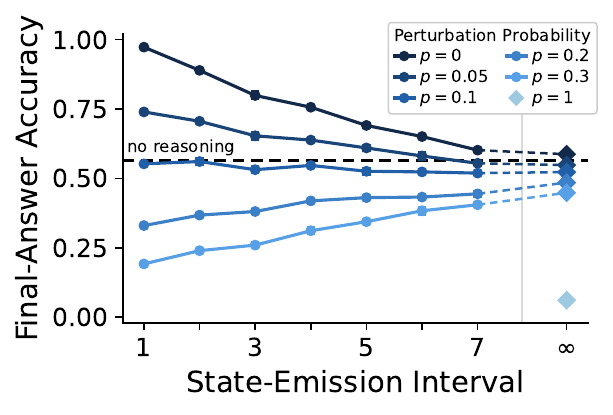}
\includegraphics[width=0.4\linewidth]{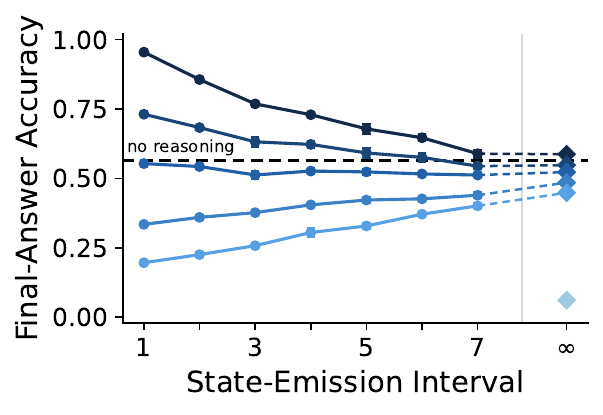}
\caption{\textbf{Reasoning with more frequent state emissions results in greater accuracy with fewer samples, but such reasoning is less robust to external perturbations.} \textbf{Left:} Accuracy of the final answer for different emission intervals and perturbation probabilities on a held-out test set. A different model was trained for each state-emission interval. Each model was trained on 50,000 samples. \textbf{Right:} Accuracy when models are trained with a fixed number of tokens is nearly identical to when models are trained with a fixed number of samples (left). Configurations corresponding to $k=\infty$ and those with no reasoning get 50,000 training samples, while smaller state-emission intervals get fewer samples because their reasoning traces contain more tokens per sample. For example, $k=1$ was trained on only 24,536 samples. In both left and right panels, without perturbation, using a small emission interval $k=1$ is optimal. With perturbation, the trend flips and the $k=\infty$ condition is optimal. Error-bars are standard error of the mean over three replicate experiments with different sets of train and test samples. They are mostly too small to see in the plot. Details in~\cref{app:interp}.}
\label{fig:math_eval_acc}
\end{figure}

First,~\cref{fig:math_eval_acc} shows that the $k=\infty$ reasoning traces contribute almost nothing to the final-answer accuracy compared to direct prediction. In~\cref{fig:math_eval_acc} (left), a model trained to predict without reasoning obtains 0.564 accuracy, which is almost the same as that of the model with $k=\infty$ (0.587). This may seem surprising in light of evidence that LLMs obtain better accuracy when they use reasoning. Our finding is not inconsistent with this evidence. We are suggesting that reasoning is more accurate than direct prediction if there is some De Bruijn structure in the training data. Indeed, reasoning with state-emission interval $k=1$ has a significantly better accuracy of 0.973 in~\cref{fig:math_eval_acc} (left).

Second, we checked that the proportion of correct answers with incorrect reasoning increases monotonically with the state-emission interval $k$ and is maximized at $k=\infty$. When there are no perturbations, models that emit the state frequently ($k=1$) do not produce correct answers via incorrect steps. For all perturbation probabilities $p$, the fraction of correct answers that are produced by incorrect reasoning is at most 3\%. This is in contrast with the behavior of the model that does not emit the state at all. When there are no perturbations, 16.6\% of correct answers have incorrect reasoning while, with perturbations, this number can be as high as 35\%. Reasoning without state emissions is thus only loosely coupled to the final answer. This explains the discrepancy between trends in the derivation accuracy in~\cref{fig:math_eval_trends} (left) and the final-answer accuracy in~\cref{fig:math_eval_acc} (left).

\begin{figure}[!htpb]
\centering
\includegraphics[width=0.4\linewidth]{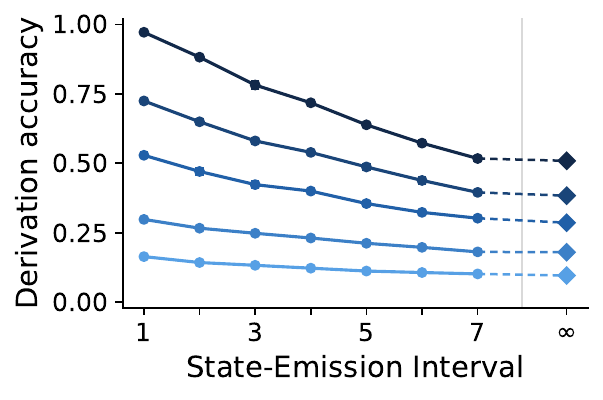}
\includegraphics[width=0.4\linewidth]{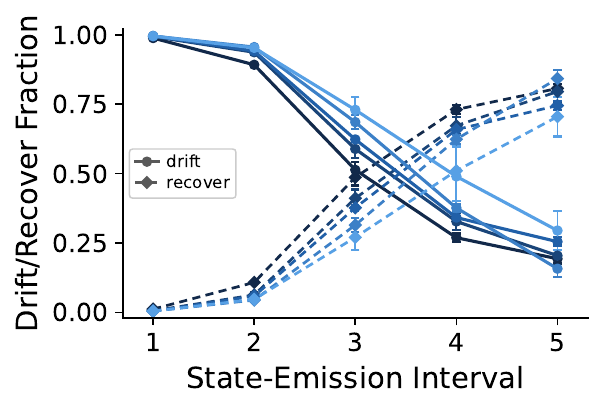}
\caption{\textbf{Frequent state emissions result in faithful, step-by-step reasoning instead of pattern recognition.} \textbf{Left:} For all perturbation probabilities, correctness of the intermediate sequence of reasoning steps (the ``derivation accuracy'') decays as the state-emission interval grows. When we compare to the trends in final-answer accuracy shown in~\cref{fig:math_eval_acc}, we see that models with large state-emission intervals tend to get the correct answer with incorrect derivations when their reasoning is perturbed. \textbf{Right:} Given that a model makes a mistake in a state prediction, future predictions can either ``drift'' or ``recover''. A model drifts when its current state is correct conditioned on the recent past context (up to the previous state) but incorrect given the entire past sequence of operations (ignoring previous states). This means it propagated an error from a previous state. It recovers if the opposite is true, i.e., when it corrects a mistake in a previous incorrect state. As the emission interval increases, it is more likely for a state prediction to recover than drift into further incorrect states. Both curves add to one. Error-bars are standard error over three replicate experiments with different sets of train and test samples. They are mostly too small to see in the plot. All models were trained on 50,000 samples. Details in~\cref{app:interp}}
\label{fig:math_eval_trends}
\end{figure}

And third, for large state-emission intervals $k$, the model stops deriving each intermediate equation (state) from the one before it. \cref{fig:math_eval_trends} (right) shows that at small $k$, an error in one state is inherited by later states, i.e., its reasoning drifts. As $k$ increases, the model is more likely to recover from an incorrect previous state, i.e., an error in one state emission is corrected in future states.

Together, our results show that for small $k$, each state is conditioned on primarily the most recent tokens after the previous state emission. The final answer is a result of this sequence of local steps. For larger $k$, each state and the final answer depends on learned patterns from the entire global context. This is why reasoning is less robust but more accurate and faithful for small $k$ (more De Bruijn structure), and it is more robust but not as accurate for large $k$ (less De Bruijn structure).

\begin{remark}
We also explored two alternative hypotheses about why configurations with larger state-emission intervals $k$ are more robust to perturbations compared to those with smaller state-emission intervals (\cref{fig:math_eval_acc} and \cref{fig:math_eval_trends}). For example, it is possible that these large $k$ conditions reason step-by-step but skip over intermediate perturbations. Or perhaps, the LLM implicitly replaces perturbations with the corresponding correct step. To test these, we also computed the ``ignored accuracy'' and the ``replaced accuracy''. They are computed in a similar fashion to the derivation accuracy except that the ignored accuracy does not take into account any steps that were added due to perturbation, and the replaced accuracy replaces steps due to perturbation with the canonically correct step. \cref{fig:ignored_replaced} shows that neither accuracy measure recovers the trend seen in the final-answer accuracy.
\end{remark}

\begin{wrapfigure}[17]{R}{0.4\linewidth}
\vspace*{-2ex}
\begin{framed}
\footnotesize
\noindent Story: Nicholas entered the hostel dorm. Julia entered the hostel dorm. Nicholas moved the harmonica to the canvas bag, which is also located in the hostel dorm. Nicholas left the hostel dorm. Julia moved the harmonica to the leather pouch, which is also located in the hostel dorm.

\noindent Question: In which container will Nicholas search for the harmonica?
\end{framed}
\caption{\textbf{An example story and question.} Answering the question requires tracking the beliefs of Nicholas and Julia about the location of the harmonica. The actual questions used for evaluation were much more substantial than this, with 20 actions and multiple objects to track.}
\label{fig:storytask}
\end{wrapfigure}

\paragraph{Story task}
We next fine-tune Qwen2.5-1.5B-Instruct \citep{qwenQwen25TechnicalReport2025} on a natural language task where the model must answer questions about stories. The stories, questions and answers are algorithmically generated based on the approach by~\citet{sclarExploreTheoryMind2024}. To answer these questions using step-by-step reasoning, the model must correctly track the beliefs of different people about certain objects and how their actions change these beliefs. An example task is shown in~\cref{fig:storytask}. We studied models that are (i) not fine-tuned (default reasoning), (ii) fine-tuned to produce the final answer directly, (iii) trained to emit a short reasoning trace describing only the key sentences in the story relevant to the final answer, (iv) trained to restate every step in the story followed by the current state of all relevant variables (people, their beliefs, relevant objects). We also studied a length-matched control to the state-tracking models in (iv) where models restate each line of the story and, instead of producing the state afterwards, only describe the local changes implemented by each step. As in the previous section, we created data with the different state-emission intervals. \cref{app:story} provides more examples and details.

\begin{figure}[!btph]
\centering
\includegraphics[width=0.4\linewidth]{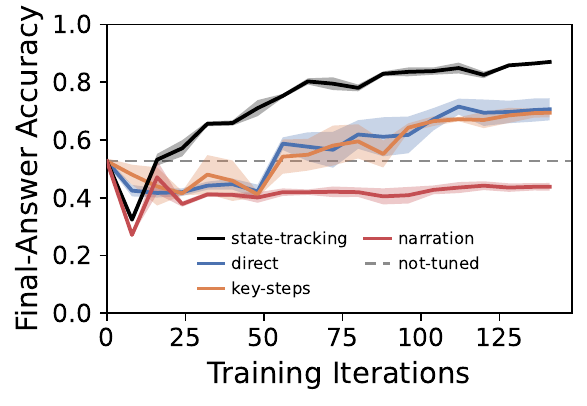}
\includegraphics[width=0.4\linewidth]{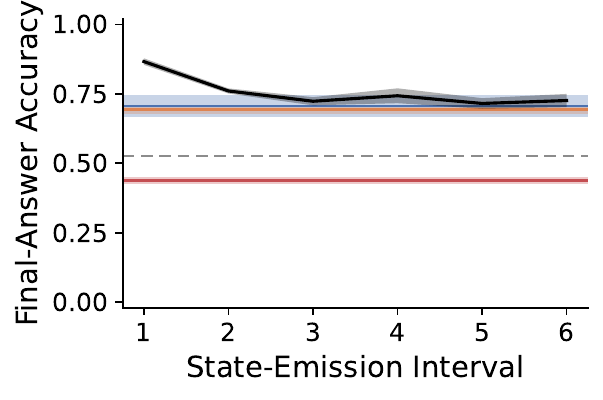}
\caption{\textbf{Reasoning with frequent state-emissions is more accurate than controls and reasoning with fewer state-emissions}. \textbf{Left:} The accuracy of the final answer increases steadily during fine-tuning. When the model emits every intermediate state (state-tracking), it performs better than the other configurations. \textbf{Right:} As the state-emission interval increases, accuracy declines steadily to be the same as not using any reasoning at all. Shading denotes the standard error over three replicate experiments. These trends are consistent with those in~\cref{fig:math_eval_acc} on the equation-solving task.}
\label{fig:tom_results}
\end{figure}
We find that reasoning with state-tracking results in higher accuracy than all other configurations (\cref{fig:tom_results}, left). Narration models where reasoning traces are roughly the same length as those of state-tracking  perform the worst, 0.438 accuracy for the former with a median reasoning length of 648 versus 0.866 for the latter with a median reasoning length of 852. Reasoning with key sentences obtains an accuracy of 0.694 with a median reasoning length of only 38 tokens. This indicates that our findings are not just an artifact of increasing the reasoning length by including more states. In addition, as the state-emission interval is increased, the accuracy deteriorates to that of no reasoning at all (\cref{fig:tom_results}, right). This supports the conclusion drawn from the equation-solving task in~\cref{fig:math_eval_acc}, where models with larger state-emission intervals infer next tokens using pattern recognition rather than step-by-step reasoning.

\section{real-world reasoning traces exhibit de bruijn structure.}
\begin{figure}[!btph]
\centering
\includegraphics[width=0.325\linewidth]{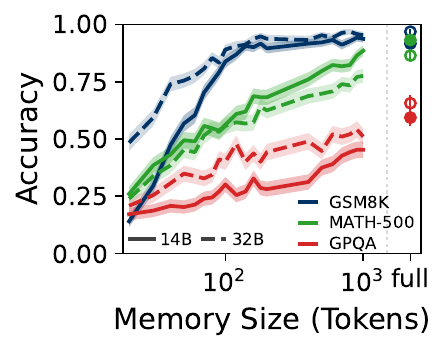}
\includegraphics[width=0.325\linewidth]{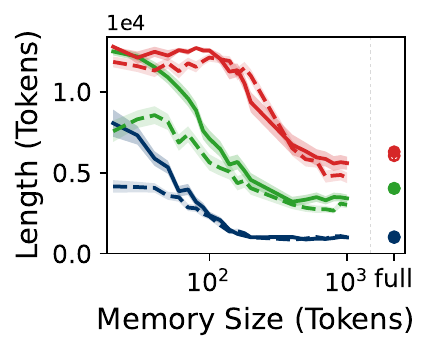}
\includegraphics[width=0.325\linewidth]{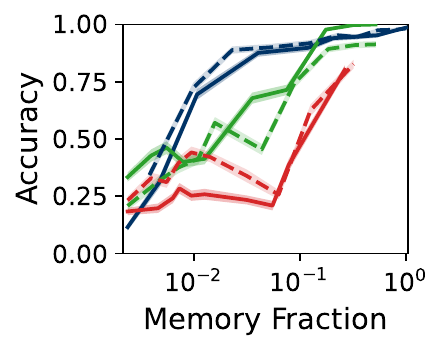}
\caption{\textbf{Qwen3-14B and Qwen3-32B perform quite well on reasoning benchmarks like GSM8K, MATH-500 and GPQA-Diamond in spite of attending only to the prompt and a few recent few tokens (denoted as ``memory size'') during inference.} \textbf{Left:} With a memory size of only 70, 400 and 900 tokens on GSM8K, MATH-500 and GPQA-Diamond respectively, Qwen3-14B reaches 75\% of its full-context accuracy. Qwen3-32B reaches the same percentage of its full-context accuracy with equal or smaller memory (40, 400 and 600 tokens). As a control, we evaluated Qwen3-14B on these benchmarks while only attending to the prompt and a random subset of past reasoning tokens of the same sizes (70, 400 and 900 tokens). It reached only 0\%, 2\% and 27\% of its full-context accuracy. The third GPQA-Diamond result corresponds to an absolute accuracy of 0.161, below the chance level of 0.25. \textbf{Middle:} These memory sizes are a small fraction of the length of correct reasoning. For both models the length of correct reasoning exceeds 992, 4027 and 6054 tokens for GSM8K, MATH-500 and GPQA-Diamond respectively (incorrect reasoning is longer). \textbf{Right:} Accuracy as a function of memory fraction $c/L_{i,c}$ where $c$ is the memory size, $L_{i,c}$ is the generated length of reasoning on problem $i$ for memory size $c$. All memory fractions, indexed by $(i,c)$ pairs, are quantile-binned into ten equal-count bins per model and benchmark. The x-value of each point is the median memory fraction across a bin and the y-value is the mean accuracy across that bin. Both LLMs achieve good accuracy at only 15\% memory fraction. Error bars in the plots denote standard error over 192 problems per benchmark. Details in~\cref{app:knockout}}
\label{fig:knockout_exp}
\end{figure}
Our experiments in \cref{fig:math_eval_acc} and \cref{fig:tom_results} indicate that models should reason with a moderate amount of De Bruijn structure to balance the accuracy than and robustness of their reasoning. We would like to check whether well-trained LLMs could be exploiting the De Bruijn structure in real-world tasks, if it exists. However, it is difficult to precisely quantify the De Bruijn structure in real-world LLM reasoning traces. This is why, in the equation-solving task in \cref{fig:math_eval_acc}, we used explicit intermediate state variables to induce De Bruijn structure. In real-world tasks, reasoning traces often do not represent the state explicitly, and a particular state may only approximately decouple past and future reasoning. Regardless, any LLM that reasons with De Bruijn structure should be able to continue its reasoning trace from a short window of recent context.

We evaluate Qwen3-14B and Qwen3-32B \citep{yangQwen3TechnicalReport2025} on mathematical reasoning benchmarks like GSM8K and MATH-500, and the graduate-level scientific reasoning benchmark GPQA-Diamond. We apply a mask so that each LLM only attends to the prompt and a ``memory'' that consists of its most recent tokens. With a memory size equal to about one-tenth of the mean trace length, both Qwen models retain roughly 90\% of the full reasoning accuracy on GSM8K and 80\% of the accuracy on MATH-500. On GPQA-Diamond, the 14B model retains only 65\% of its full-context accuracy while the 32B model retains 80\%, so the larger model needs less working memory on the harder task (\cref{fig:knockout_exp}, left, middle). GSM8K accuracy saturates near 150--180 tokens for both models. This gives us an indication of the amount of preceding tokens a single reasoning step requires on this task. Furthermore, \cref{fig:knockout_exp} (right) shows that, with memory size that is $\sim$15\% of each problem's reasoning length, both models achieve greater than 0.9 accuracy on GSM8K, greater than 0.85 accuracy on MATH-500 and greater than 0.6 on GPQA-Diamond (estimated by log-linear interpolation in \cref{fig:knockout_exp}, right). Our results show that LLMs can solve mathematical reasoning benchmarks by relying on only their prompt and a local context. When memory is limited, the larger model Qwen3-32B generates fewer tokens to achieve the same accuracy as the smaller model Qwen3-14B (\cref{fig:knockout_exp}, middle). Limiting memory has a smaller effect on larger and more capable models, perhaps suggesting that they can better exploit De Bruijn structure implicit in the data. Our results are consistent with \citet{aghajohariMarkovianThinkerArchitectureAgnostic2025}, who found that off-the-shelf reasoning models retain most of their accuracy when they condition only on the prompt and recent tokens.

\section{related work and discussion}

\textbf{Pattern recognition vs.\@ step-by-step reasoning.}
Generalization is the ability of a model to apply patterns learned from training data to new, unseen data. Classical analyses of this ability, using PAC learning and VC theory to characterize the capacity of the hypothesis class~\citep{vapnik1998statistical,Valiant1984-dx}, mutual information between samples and learned hypotheses~\citep{xu2017information}, or algorithmic stability to characterize the sensitivity of the trained hypothesis to the sample set~\citep{bousquet2002stability}, all rest on the same premise: new data are drawn from the same distribution as the training data. Such generalization to the training distribution is the formal instantiation of the intuitive idea that similar inputs should yield similar outputs~\citep{belkin2019reconciling,cover1967nearest}, i.e., pattern recognition~\citep{minsky1969introduction}. LLMs that perform reasoning exhibit a very different type of generalization. They can piece together evidence, either from training data or from the previous context, in complex ways to solve tasks not seen during pre-training~\citep{weiChainofThoughtPromptingElicits2022}. There is a growing body of literature that investigates this and related phenomena.

Some researchers have argued that reasoning is beneficial because it allows LLMs to make inferences with fewer parameters~\citep{joshi2025theorylearningautoregressivechain}. Others have studied how LLMs extend the size of the computational circuit by chaining operations~\citep{feng2023revealingmysterychainthought,yehudai2025compositionalreasoningtransformersrnns,stroblWhatFormalLanguages2024}. In addition, \citet{prystawskiWhyThinkStep2023a} show that LLMs can traverse Bayesian belief networks by chaining together a sequence of local transitions. There are also a large number of techniques to improve reasoning \citep{weiChainofThoughtPromptingElicits2022, wang2023selfconsistencyimproveschainthought,shinn2023reflexionlanguageagentsverbal,Besta_2024,muennighoff2025s1simpletesttimescaling,yao2023treethoughtsdeliberateproblem,li2025compressingchainofthoughtllmsstep}. As well as works that identify the ideal length of reasoning \citep{sui2025stopoverthinkingsurveyefficient,wu2025lessunderstandingchainofthoughtlength,nyeShowYourWork2021}.

Our work is focused on two key questions: (i) why reasoning emerges in LLMs trained by next-token prediction, and (ii) whether reasoning is different from pattern recognition. First, we show that reasoning can emerge in models that are trained to predict the next token when the data have a De Bruijn structure. Second, we argue that, depending on the structure in the training data, LLMs can implement a spectrum of inference mechanisms, from step-by-step reasoning to statistical pattern recognition. This versatility enables LLMs to compose a precise sequence of reasoning steps for well-structured tasks but also defer to pattern recognition when tasks are ambiguous.

\textbf{Learning simplified models of language with transformers.}
There is work on training transformers on data drawn from a Markov chain to estimate the transition kernel~\citep{makkuvaAttentionMarkovFramework2025,edelmanEvolutionStatisticalInduction2024}. In addition, transformers can represent $n$-gram models exactly~\citep{sveteTransformersCanRepresent2024}. If reasoning has a De Bruijn structure, then there exists a state that makes future reasoning independent of the past. Our results in~\cref{fig:knockout_exp} show that well-trained LLMs can predict quite accurately on mathematical benchmarks in spite of being given access to only a small number of recent tokens and the prompt. While Markov chains and hidden Markov models may be useful models of natural language, our results suggest that De Bruijn graphs are a good model to study questions related to reasoning in LLMs. While there are likely many tasks which are not solvable by reasoning traces with De Bruijn structure, \citet{bertran2026fits} show that the generalizable results of real-world machine learning research are often highly compressible into short token sequences. Conditioned on these sequences, future research is mostly independent of the specific details of earlier results, inducing a De Bruijn structure on the reasoning traces for machine learning research.

\textbf{LLMs are often trained on structured reasoning traces}, e.g., to restate the preceding context or to maintain scratchpads that compress past reasoning~\citep{nyeShowYourWork2021,yanInftyThinkBreakingLength2025,zhouMEM1LearningSynergize2025,aghajohariMarkovianThinkerArchitectureAgnostic2025}. Sliding window attention reinforces an existing tendency of models to attend to their recent context~\citep{liuLostMiddleHow2024,jiangMistral7B2023}. Researchers have also explored using an ``overseer'' LLM to delegate tasks to sub-agents~\citep{zhangChainAgentsLarge2024,HowWeBuilt}. The overseer can perform future reasoning using the sub-agent's results, without relying on a lengthy past context. Thus, the above techniques may be understood as introducing a ``state'' in the reasoning. Any set of reasoning traces can be modeled by a De Bruijn graph $B(V, c)$. If traces have no structure, the context length $c$ in the graph is simply the length of the reasoning trace itself. The smaller the value of $c$, the smaller the complexity of the state. In this sense, techniques that induce a specific kind of state may be interpreted as reducing the value of $c$ to improve performance. In some cases, this state representation can be observed in the activations of the model~\citep{liEmergentWorldRepresentations2022,ramakrishnanMultimodalModelsImagine2026d}. Similar to our story-task in \cref{fig:tom_results}, \citet{bigelow2026stories} also use stories to argue that the trajectories of belief-states of LLMs as they process stories evolve on low-dimensional manifolds.

\textbf{De Bruijn graphs in genomics.} There has been interest in using transformer-based models to generate realistic genome sequences to engineer biological systems~\citep{nguyenSequenceModelingDesign2024}. These models identify common structures in genome sequences and can reproduce these motifs in new sequences. These sequences are known to have De Bruijn structure, and De Bruijn graphs are a common tool used to compress genome sequences into a smaller set of transitions between $c$-mers~\citep{compeauWhyAreBruijn2011}. It is computationally efficient to sample realistic genome sequences by traversing paths along these De Bruijn graphs~\citep{bankevichMultiplexBruijnGraphs2022}. Our results suggest that transformer-based models are successful at generating genome sequences because they have De Bruijn structure. And perhaps they are less successful on sequences where this structure is not explicitly represented. The latter case might provide an opportunity to apply reasoning to problems in biology. It may be possible to train a model to generate intermediate tokens with De Bruijn structure that can subsequently be expanded into these unstructured sequences. 


\newpage
\subsection*{AI use statement}


In this work, we primarily used generative AI tools for writing code, generating synthetic data and proof-checking. We also used generative AI to help us check for mistakes in the grammar and polish some of the writing in one or two paragraphs. We have not used generative AI tools for other tasks such as finding related works, to draft paper sections, and to prove mathematical claims. We have reviewed all AI-assisted work manually. We checked all AI-generated code manually and ran sanity checks and test experiments. We manually checked AI-generated datasets as well as the raw outputs of the models we were evaluating on these datasets. We take responsibility for all the final content of this work, including text, claims or artifacts produced with the aid of generative AI.


\subsection*{Ethics statement}

This paper focuses on theoretical results. We do not believe it raises any ethical concerns.

\subsection*{Reproducibility statement}


The Appendix provides extensive details of the synthetic tasks, architectures and training procedures used in this paper. It also provides detailed proofs of all mathematical results. We will make all source-code for reproducing the numerical experiments discussed in this paper public.




\bibliographystyle{iclr2027_conference}
\bibliography{iclr2027_conference,pratik}

@article{bigelow2026stories,
  title={Stories in Space: In-Context Learning Trajectories in Conceptual Belief Space},
  author={Bigelow, Eric and Sarfati, Rapha{\"e}l and Wurgaft, Daniel and Lewis, Owen and McGrath, Thomas and Merullo, Jack and Geiger, Atticus and Lubana, Ekdeep Singh},
  journal={arXiv preprint arXiv:2605.12412},
  year={2026}
}

@article{vafa2024evaluating,
  title={Evaluating the world model implicit in a generative model},
  author={Vafa, Keyon and Chen, Justin Y and Rambachan, Ashesh and Kleinberg, Jon and Mullainathan, Sendhil},
  journal={Advances in Neural Information Processing Systems},
  volume={37},
  pages={26941--26975},
  year={2024}
}

@inproceedings{weiChainofThoughtPromptingElicits2022,
  title = {Chain-of-{{Thought Prompting Elicits Reasoning}} in {{Large Language Models}}},
  booktitle = {Advances in {{Neural Information Processing Systems}}},
  author = {Wei, Jason and Wang, Xuezhi and Schuurmans, Dale and Bosma, Maarten and {ichter}, brian and Xia, Fei and Chi, Ed and Le, Quoc V and Zhou, Denny},
  year = 2022,
  volume = {35},
  pages = {24824--24837},
  publisher = {Curran Associates, Inc.},
  doi = {10.52202/068431-1800},
  urldate = {2026-09-02}
}

@article{guoDeepSeekR1IncentivizesReasoning2025a,
  title = {{{DeepSeek-R1}} Incentivizes Reasoning in {{LLMs}} through Reinforcement Learning},
  author = {Guo, Daya and Yang, Dejian and Zhang, Haowei and Song, Junxiao and Wang, Peiyi and Zhu, Qihao and Xu, Runxin and Zhang, Ruoyu and Ma, Shirong and Bi, Xiao and Zhang, Xiaokang and Yu, Xingkai and Wu, Yu and Wu, Z. F. and Gou, Zhibin and Shao, Zhihong and Li, Zhuoshu and Gao, Ziyi and Liu, Aixin and Xue, Bing and Wang, Bingxuan and Wu, Bochao and Feng, Bei and Lu, Chengda and Zhao, Chenggang and Deng, Chengqi and Ruan, Chong and Dai, Damai and Chen, Deli and Ji, Dongjie and Li, Erhang and Lin, Fangyun and Dai, Fucong and Luo, Fuli and Hao, Guangbo and Chen, Guanting and Li, Guowei and Zhang, H. and Xu, Hanwei and Ding, Honghui and Gao, Huazuo and Qu, Hui and Li, Hui and Guo, Jianzhong and Li, Jiashi and Chen, Jingchang and Yuan, Jingyang and Tu, Jinhao and Qiu, Junjie and Li, Junlong and Cai, J. L. and Ni, Jiaqi and Liang, Jian and Chen, Jin and Dong, Kai and Hu, Kai and You, Kaichao and Gao, Kaige and Guan, Kang and Huang, Kexin and Yu, Kuai and Wang, Lean and Zhang, Lecong and Zhao, Liang and Wang, Litong and Zhang, Liyue and Xu, Lei and Xia, Leyi and Zhang, Mingchuan and Zhang, Minghua and Tang, Minghui and Zhou, Mingxu and Li, Meng and Wang, Miaojun and Li, Mingming and Tian, Ning and Huang, Panpan and Zhang, Peng and Wang, Qiancheng and Chen, Qinyu and Du, Qiushi and Ge, Ruiqi and Zhang, Ruisong and Pan, Ruizhe and Wang, Runji and Chen, R. J. and Jin, R. L. and Chen, Ruyi and Lu, Shanghao and Zhou, Shangyan and Chen, Shanhuang and Ye, Shengfeng and Wang, Shiyu and Yu, Shuiping and Zhou, Shunfeng and Pan, Shuting and Li, S. S. and Zhou, Shuang and Wu, Shaoqing and Yun, Tao and Pei, Tian and Sun, Tianyu and Wang, T. and Zeng, Wangding and Liu, Wen and Liang, Wenfeng and Gao, Wenjun and Yu, Wenqin and Zhang, Wentao and Xiao, W. L. and An, Wei and Liu, Xiaodong and Wang, Xiaohan and Chen, Xiaokang and Nie, Xiaotao and Cheng, Xin and Liu, Xin and Xie, Xin and Liu, Xingchao and Yang, Xinyu and Li, Xinyuan and Su, Xuecheng and Lin, Xuheng and Li, X. Q. and Jin, Xiangyue and Shen, Xiaojin and Chen, Xiaosha and Sun, Xiaowen and Wang, Xiaoxiang and Song, Xinnan and Zhou, Xinyi and Wang, Xianzu and Shan, Xinxia and Li, Y. K. and Wang, Y. Q. and Wei, Y. X. and Zhang, Yang and Xu, Yanhong and Li, Yao and Zhao, Yao and Sun, Yaofeng and Wang, Yaohui and Yu, Yi and Zhang, Yichao and Shi, Yifan and Xiong, Yiliang and He, Ying and Piao, Yishi and Wang, Yisong and Tan, Yixuan and Ma, Yiyang and Liu, Yiyuan and Guo, Yongqiang and Ou, Yuan and Wang, Yuduan and Gong, Yue and Zou, Yuheng and He, Yujia and Xiong, Yunfan and Luo, Yuxiang and You, Yuxiang and Liu, Yuxuan and Zhou, Yuyang and Zhu, Y. X. and Huang, Yanping and Li, Yaohui and Zheng, Yi and Zhu, Yuchen and Ma, Yunxian and Tang, Ying and Zha, Yukun and Yan, Yuting and Ren, Z. Z. and Ren, Zehui and Sha, Zhangli and Fu, Zhe and Xu, Zhean and Xie, Zhenda and Zhang, Zhengyan and Hao, Zhewen and Ma, Zhicheng and Yan, Zhigang and Wu, Zhiyu and Gu, Zihui and Zhu, Zijia and Liu, Zijun and Li, Zilin and Xie, Ziwei and Song, Ziyang and Pan, Zizheng and Huang, Zhen and Xu, Zhipeng and Zhang, Zhongyu and Zhang, Zhen},
  year = 2025,
  month = sep,
  journal = {Nature},
  volume = {645},
  number = {8081},
  pages = {633--638},
  publisher = {Nature Publishing Group},
  issn = {1476-4687},
  doi = {10.1038/s41586-025-09422-z},
  urldate = {2026-09-02},
  copyright = {2025 The Author(s)},
  langid = {english}
}

@article{glazerFrontierMathBenchmarkEvaluating2025,
  title = {{{FrontierMath}}: {{A Benchmark}} for {{Evaluating Advanced Mathematical Reasoning}} in {{AI}}},
  shorttitle = {{{FrontierMath}}},
  author = {Glazer, Elliot and Erdil, Ege and Besiroglu, Tamay and Chicharro, Diego and Chen, Evan and Gunning, Alex and Olsson, Caroline Falkman and Denain, Jean-Stanislas and Ho, Anson and Santos, Emily de Oliveira and J{\"a}rviniemi, Olli and Barnett, Matthew and Sandler, Robert and Vrzala, Matej and Sevilla, Jaime and Ren, Qiuyu and Pratt, Elizabeth and Levine, Lionel and Barkley, Grant and Stewart, Natalie and Grechuk, Bogdan and Grechuk, Tetiana and Enugandla, Shreepranav Varma and Wildon, Mark},
  year = 2025,
  month = dec,
  number = {arXiv:2411.04872},
  eprint = {2411.04872},
  primaryclass = {cs.AI},
  journal = {arXiv},
  doi = {10.48550/arXiv.2411.04872},
  urldate = {2026-09-02},
  archiveprefix = {arXiv}
}

@article{hubertOlympiadlevelFormalMathematical2026,
  title = {Olympiad-Level Formal Mathematical Reasoning with Reinforcement Learning},
  author = {Hubert, Thomas and Mehta, Rishi and Sartran, Laurent and Horv{\'a}th, Mikl{\'o}s Z. and {\v Z}u{\v z}i{\'c}, Goran and Wieser, Eric and Huang, Aja and Schrittwieser, Julian and Schroecker, Yannick and Masoom, Hussain and Bertolli, Ottavia and Zahavy, Tom and Mandhane, Amol and Yung, Jessica and Beloshapka, Iuliya and Ibarz, Borja and Veeriah, Vivek and Yu, Lei and Nash, Oliver and Lezeau, Paul and Mercuri, Salvatore and S{\"o}nne, Calle and Mehta, Bhavik and Davies, Alex and Zheng, Daniel and Pedregosa, Fabian and Li, Yin and {von Glehn}, Ingrid and Rowland, Mark and Albanie, Samuel and Velingker, Ameya and Schmitt, Simon and Lockhart, Edward and Hughes, Edward and Michalewski, Henryk and Sonnerat, Nicolas and Hassabis, Demis and Kohli, Pushmeet and Silver, David},
  year = 2026,
  month = mar,
  journal = {Nature},
  volume = {651},
  number = {8106},
  pages = {607--613},
  publisher = {Nature Publishing Group},
  issn = {1476-4687},
  doi = {10.1038/s41586-025-09833-y},
  urldate = {2026-09-02},
  copyright = {2025 The Author(s)},
  langid = {english}
}

@inproceedings{jimenezSWEbenchCanLanguage2023,
  title = {{{SWE-bench}}: {{Can Language Models Resolve Real-world Github Issues}}?},
  shorttitle = {{{SWE-bench}}},
  booktitle = {The {{Twelfth International Conference}} on {{Learning Representations}}},
  author = {Jimenez, Carlos E. and Yang, John and Wettig, Alexander and Yao, Shunyu and Pei, Kexin and Press, Ofir and Narasimhan, Karthik R.},
  year = 2023,
  month = oct,
  urldate = {2026-09-02},
  langid = {english}
}

@inproceedings{lewkowyczSolvingQuantitativeReasoning2022,
  title = {Solving {{Quantitative Reasoning Problems}} with {{Language Models}}},
  booktitle = {Advances in {{Neural Information Processing Systems}}},
  author = {Lewkowycz, Aitor and Andreassen, Anders Johan and Dohan, David and Dyer, Ethan and Michalewski, Henryk and Ramasesh, Vinay Venkatesh and Slone, Ambrose and Anil, Cem and Schlag, Imanol and {Gutman-Solo}, Theo and Wu, Yuhuai and Neyshabur, Behnam and {Gur-Ari}, Guy and Misra, Vedant},
  year = 2022,
  month = oct,
  urldate = {2026-09-02},
  langid = {english}
}

@inproceedings{nadgirHowDoesChain2026b,
  title = {How Does {{Chain}} of {{Thought}} Decompose Complex Tasks?},
  booktitle = {Forty-Third {{International Conference}} on {{Machine Learning}}},
  author = {Nadgir, Amrut and Balasubramanian, Vijay and Chaudhari, Pratik},
  year = 2026,
  month = jun,
  urldate = {2026-09-02},
  langid = {english}
}

@inproceedings{cagnettaDerivingNeuralScaling2026c,
  title = {Deriving {{Neural Scaling Laws}} from the {{Statistics}} of {{Natural Language}}},
  booktitle = {Forty-Third {{International Conference}} on {{Machine Learning}}},
  author = {Cagnetta, Francesco and Raventos, Allan and Ganguli, Surya and Wyart, Matthieu},
  year = 2026,
  month = jun,
  urldate = {2026-09-02},
  langid = {english}
}

@inproceedings{prystawskiWhyThinkStep2023a,
  title = {Why Think Step by Step? {{Reasoning}} Emerges from the Locality of Experience},
  shorttitle = {Why Think Step by Step?},
  booktitle = {Thirty-Seventh {{Conference}} on {{Neural Information Processing Systems}}},
  author = {Prystawski, Ben and Li, Michael Y. and Goodman, Noah},
  year = 2023,
  month = nov,
  urldate = {2026-09-04},
  langid = {english}
}

@inproceedings{taoScalingLawsVocabulary2024,
  title = {Scaling {{Laws}} with {{Vocabulary}}: {{Larger Models Deserve Larger Vocabularies}}},
  shorttitle = {Scaling {{Laws}} with {{Vocabulary}}},
  booktitle = {The {{Thirty-eighth Annual Conference}} on {{Neural Information Processing Systems}}},
  author = {Tao, Chaofan and Liu, Qian and Dou, Longxu and Muennighoff, Niklas and Wan, Zhongwei and Luo, Ping and Lin, Min and Wong, Ngai},
  year = 2024,
  month = nov,
  urldate = {2026-09-05},
  langid = {english}
}

@article{nyeShowYourWork2021,
  title = {Show {{Your Work}}: {{Scratchpads}} for {{Intermediate Computation}} with {{Language Models}}},
  shorttitle = {Show {{Your Work}}},
  author = {Nye, Maxwell and Andreassen, Anders Johan and {Gur-Ari}, Guy and Michalewski, Henryk and Austin, Jacob and Bieber, David and Dohan, David and Lewkowycz, Aitor and Bosma, Maarten and Luan, David and Sutton, Charles and Odena, Augustus},
  year = 2021,
  month = nov,
  number = {arXiv:2112.00114},
  eprint = {2112.00114},
  primaryclass = {cs.LG},
  journal = {arXiv},
  doi = {10.48550/arXiv.2112.00114},
  urldate = {2026-09-05},
  archiveprefix = {arXiv}
}

@article{sclarExploreTheoryMind2024,
  title = {Explore {{Theory}} of {{Mind}}: {{Program-guided}} Adversarial Data Generation for Theory of Mind Reasoning},
  shorttitle = {Explore {{Theory}} of {{Mind}}},
  author = {Sclar, Melanie and Yu, Jane and {Fazel-Zarandi}, Maryam and Tsvetkov, Yulia and Bisk, Yonatan and Choi, Yejin and Celikyilmaz, Asli},
  year = 2024,
  month = dec,
  number = {arXiv:2412.12175},
  eprint = {2412.12175},
  primaryclass = {cs.LG},
  journal = {arXiv},
  doi = {10.48550/arXiv.2412.12175},
  urldate = {2026-09-16},
  archiveprefix = {arXiv}
}

@book{sejnowskiDeepLearningRevolution2018,
  title = {The Deep Learning Revolution.},
  author = {Sejnowski, Terrence J.},
  year = 2018,
  series = {The Deep Learning Revolution.},
  pages = {x, 342},
  publisher = {The MIT Press},
  address = {Cambridge,  MA,  US},
  doi = {10.7551/mitpress/11474.001.0001},
  isbn = {9780262038034 (Hardcover)}
}

@article{qwenQwen25TechnicalReport2025,
  title = {Qwen2.5 {{Technical Report}}},
  author = {Qwen and Yang, An and Yang, Baosong and Zhang, Beichen and Hui, Binyuan and Zheng, Bo and Yu, Bowen and Li, Chengyuan and Liu, Dayiheng and Huang, Fei and Wei, Haoran and Lin, Huan and Yang, Jian and Tu, Jianhong and Zhang, Jianwei and Yang, Jianxin and Yang, Jiaxi and Zhou, Jingren and Lin, Junyang and Dang, Kai and Lu, Keming and Bao, Keqin and Yang, Kexin and Yu, Le and Li, Mei and Xue, Mingfeng and Zhang, Pei and Zhu, Qin and Men, Rui and Lin, Runji and Li, Tianhao and Tang, Tianyi and Xia, Tingyu and Ren, Xingzhang and Ren, Xuancheng and Fan, Yang and Su, Yang and Zhang, Yichang and Wan, Yu and Liu, Yuqiong and Cui, Zeyu and Zhang, Zhenru and Qiu, Zihan},
  year = 2025,
  month = jan,
  number = {arXiv:2412.15115},
  eprint = {2412.15115},
  primaryclass = {cs.CL},
  journal = {arXiv},
  doi = {10.48550/arXiv.2412.15115},
  urldate = {2026-09-21},
  archiveprefix = {arXiv}
}

@article{yangQwen3TechnicalReport2025,
  title = {Qwen3 {{Technical Report}}},
  author = {Yang, An and Li, Anfeng and Yang, Baosong and Zhang, Beichen and Hui, Binyuan and Zheng, Bo and Yu, Bowen and Gao, Chang and Huang, Chengen and Lv, Chenxu and Zheng, Chujie and Liu, Dayiheng and Zhou, Fan and Huang, Fei and Hu, Feng and Ge, Hao and Wei, Haoran and Lin, Huan and Tang, Jialong and Yang, Jian and Tu, Jianhong and Zhang, Jianwei and Yang, Jianxin and Yang, Jiaxi and Zhou, Jing and Zhou, Jingren and Lin, Junyang and Dang, Kai and Bao, Keqin and Yang, Kexin and Yu, Le and Deng, Lianghao and Li, Mei and Xue, Mingfeng and Li, Mingze and Zhang, Pei and Wang, Peng and Zhu, Qin and Men, Rui and Gao, Ruize and Liu, Shixuan and Luo, Shuang and Li, Tianhao and Tang, Tianyi and Yin, Wenbiao and Ren, Xingzhang and Wang, Xinyu and Zhang, Xinyu and Ren, Xuancheng and Fan, Yang and Su, Yang and Zhang, Yichang and Zhang, Yinger and Wan, Yu and Liu, Yuqiong and Wang, Zekun and Cui, Zeyu and Zhang, Zhenru and Zhou, Zhipeng and Qiu, Zihan},
  year = 2025,
  month = may,
  number = {arXiv:2505.09388},
  eprint = {2505.09388},
  primaryclass = {cs.CL},
  journal = {arXiv},
  doi = {10.48550/arXiv.2505.09388},
  urldate = {2026-09-22},
  archiveprefix = {arXiv}
}

@article{gsm8k,
  author       = {Karl Cobbe and
                  Vineet Kosaraju and
                  others},
  title        = {Training Verifiers to Solve Math Word Problems},
  journal      = {CoRR},
  volume       = {abs/2110.14168},
  year         = {2021},
  eprinttype    = {arXiv},
  eprint       = {2110.14168}
}

@inproceedings{hendrycks2021measuringmathematicalproblemsolving,
  title = {Measuring {{Mathematical Problem Solving With}} the {{MATH Dataset}}},
  booktitle = {Thirty-Fifth {{Conference}} on {{Neural Information Processing Systems Datasets}} and {{Benchmarks Track}} ({{Round}} 2)},
  author = {Hendrycks, Dan and Burns, Collin and Kadavath, Saurav and Arora, Akul and Basart, Steven and Tang, Eric and Song, Dawn and Steinhardt, Jacob},
  year = 2021,
  month = aug,
  urldate = {2026-09-26},
  langid = {english}
}

@article{reinGPQAGraduateLevelGoogleProof2023,
  title = {{{GPQA}}: {{A Graduate-Level Google-Proof Q}}\&{{A Benchmark}}},
  shorttitle = {{{GPQA}}},
  author = {Rein, David and Hou, Betty Li and Stickland, Asa Cooper and Petty, Jackson and Pang, Richard Yuanzhe and Dirani, Julien and Michael, Julian and Bowman, Samuel R.},
  year = 2023,
  month = nov,
  journal = {arXiv},
  urldate = {2026-09-22},
  howpublished = {https://arxiv.org/abs/2311.12022v1},
  langid = {english}
}

@inproceedings{makkuvaAttentionMarkovFramework2025,
  title = {Attention with {{Markov}}: {{A Curious Case}} of {{Single-layer Transformers}}},
  shorttitle = {Attention with {{Markov}}},
  booktitle = {The {{Thirteenth International Conference}} on {{Learning Representations}}},
  author = {Makkuva, Ashok Vardhan and Bondaschi, Marco and Girish, Adway and Nagle, Alliot and Jaggi, Martin and Kim, Hyeji and Gastpar, Michael},
  year = 2024,
  month = oct,
  urldate = {2026-09-26},
  langid = {english}
}

@inproceedings{edelmanEvolutionStatisticalInduction2024,
  title = {The {{Evolution}} of {{Statistical Induction Heads}}: {{In-Context Learning Markov Chains}}},
  shorttitle = {The {{Evolution}} of {{Statistical Induction Heads}}},
  booktitle = {The {{Thirty-eighth Annual Conference}} on {{Neural Information Processing Systems}}},
  author = {Edelman, Ezra and Tsilivis, Nikolaos and Edelman, Benjamin L. and Malach, Eran and Goel, Surbhi},
  year = 2024,
  month = nov,
  urldate = {2026-09-26},
  langid = {english}
}

@article{sveteTransformersCanRepresent2024,
  title = {Transformers {{Can Represent}} \$n\$-Gram {{Language Models}}},
  author = {Svete, Anej and Cotterell, Ryan},
  year = 2024,
  month = jun,
  number = {arXiv:2404.14994},
  eprint = {2404.14994},
  primaryclass = {cs.CL},
  journal = {arXiv},
  doi = {10.48550/arXiv.2404.14994},
  urldate = {2026-09-24},
  archiveprefix = {arXiv}
}

@inproceedings{liEmergentWorldRepresentations2022,
  title = {Emergent {{World Representations}}: {{Exploring}} a {{Sequence Model Trained}} on a {{Synthetic Task}}},
  shorttitle = {Emergent {{World Representations}}},
  booktitle = {The {{Eleventh International Conference}} on {{Learning Representations}}},
  author = {Li, Kenneth and Hopkins, Aspen K. and Bau, David and Vi{\'e}gas, Fernanda and Pfister, Hanspeter and Wattenberg, Martin},
  year = 2022,
  month = sep,
  urldate = {2026-09-24},
  langid = {english}
}

@article{ramakrishnanMultimodalModelsImagine2026d,
  title = {Do Multimodal Models Imagine Electric Sheep?},
  author = {Ramakrishnan, Santhosh Kumar and Vondrick, Carl and Giryes, Raja and Kr{\"a}henb{\"u}hl, Philipp and Koltun, Vladlen},
  year = 2026,
  month = may,
  number = {arXiv:2605.09693},
  eprint = {2605.09693},
  primaryclass = {cs.CV},
  journal = {arXiv},
  doi = {10.48550/arXiv.2605.09693},
  urldate = {2026-09-24},
  archiveprefix = {arXiv}
}

@inproceedings{yanInftyThinkBreakingLength2025,
  title = {{{InftyThink}}: {{Breaking}} the {{Length Limits}} of {{Long-Context Reasoning}} in {{Large Language Models}}},
  shorttitle = {{{InftyThink}}},
  booktitle = {The {{Fourteenth International Conference}} on {{Learning Representations}}},
  author = {Yan, Yuchen and Shen, Yongliang and Liu, Yang and Jiang, Jin and Zhang, Mengdi and Shao, Jian and Zhuang, Yueting},
  year = 2025,
  month = oct,
  urldate = {2026-09-24},
  langid = {english}
}

@inproceedings{zhouMEM1LearningSynergize2025,
  title = {{{MEM1}}: {{Learning}} to {{Synergize Memory}} and {{Reasoning}} for {{Efficient Long-Horizon Agents}}},
  shorttitle = {{{MEM1}}},
  booktitle = {The {{Fourteenth International Conference}} on {{Learning Representations}}},
  author = {Zhou, Zijian and Qu, Ao and Wu, Zhaoxuan and Kim, Sunghwan and Prakash, Alok and Rus, Daniela and Low, Bryan Kian Hsiang and Liang, Paul Pu},
  year = 2025,
  month = oct,
  urldate = {2026-09-24},
  langid = {english}
}

@article{liuLostMiddleHow2024,
  title = {Lost in the {{Middle}}: {{How Language Models Use Long Contexts}}},
  shorttitle = {Lost in the {{Middle}}},
  author = {Liu, Nelson F. and Lin, Kevin and Hewitt, John and Paranjape, Ashwin and Bevilacqua, Michele and Petroni, Fabio and Liang, Percy},
  year = 2024,
  journal = {Transactions of the Association for Computational Linguistics},
  volume = {12},
  pages = {157--173},
  publisher = {MIT Press},
  address = {Cambridge, MA},
  doi = {10.1162/tacl_a_00638},
  urldate = {2026-09-24}
}

@article{jiangMistral7B2023,
  title = {Mistral {{7B}}},
  author = {Jiang, Albert Q. and Sablayrolles, Alexandre and Mensch, Arthur and Bamford, Chris and Chaplot, Devendra Singh and de las Casas, Diego and Bressand, Florian and Lengyel, Gianna and Lample, Guillaume and Saulnier, Lucile and Lavaud, L{\'e}lio Renard and Lachaux, Marie-Anne and Stock, Pierre and Scao, Teven Le and Lavril, Thibaut and Wang, Thomas and Lacroix, Timoth{\'e}e and Sayed, William El},
  year = 2023,
  month = oct,
  number = {arXiv:2310.06825},
  eprint = {2310.06825},
  primaryclass = {cs.CL},
  journal = {arXiv},
  doi = {10.48550/arXiv.2310.06825},
  urldate = {2026-09-24},
  archiveprefix = {arXiv}
}

@inproceedings{zhangChainAgentsLarge2024,
  title = {Chain of {{Agents}}: {{Large Language Models Collaborating}} on {{Long-Context Tasks}}},
  shorttitle = {Chain of {{Agents}}},
  booktitle = {The {{Thirty-eighth Annual Conference}} on {{Neural Information Processing Systems}}},
  author = {Zhang, Yusen and Sun, Ruoxi and Chen, Yanfei and Pfister, Tomas and Zhang, Rui and Arik, Sercan O.},
  year = 2024,
  month = nov,
  urldate = {2026-09-24},
  langid = {english}
}

@misc{HowWeBuilt,
    title = {How we built our multi-agent research system},
    year = 2025,
    url = {https://www.anthropic.com/engineering/multi-agent-research-system},
    language = {en},
    urldate = {2026-09-24},
    author = {Anthropic},
}

@article{bertran2026fits,
  title={What Fits (Into Few Tokens) Doesn't Overfit: Compression and Generalization in ML Research Agents},
  author={Bertran, Martin Andres and Roth, Aaron and Wu, Zhiwei Steven},
  journal={arXiv preprint arXiv:2606.11045},
  year={2026}
}

@article{joshi2025theorylearningautoregressivechain,
      title={A Theory of Learning with Autoregressive Chain of Thought}, 
      author={Nirmit Joshi and Gal Vardi and others},
      
  journal={arXiv preprint arXiv:2503.07932},
  year={2025},
      eprint={2503.07932},
      archivePrefix={arXiv},
      primaryClass={stat.ML}
}

@inproceedings{feng2023revealingmysterychainthought, author = {Feng, Guhao and Zhang, Bohang and Gu, Yuntian and Ye, Haotian and He, Di and Wang, Liwei}, title = {Towards revealing the mystery behind chain of thought: a theoretical perspective}, year = {2023}, publisher = {Curran Associates Inc.}, address = {Red Hook, NY, USA}, booktitle = {Proceedings of the 37th International Conference on Neural Information Processing Systems}, articleno = {3100}, numpages = {42}, location = {New Orleans, LA, USA}, series = {NIPS '23} }

@inproceedings{
yehudai2025compositionalreasoningtransformersrnns,
title={Compositional Reasoning with Transformers, {RNN}s, and Chain of Thought},
author={Gilad Yehudai and Noah Amsel and Joan Bruna},
booktitle={The Thirty-ninth Annual Conference on Neural Information Processing Systems},
year={2026}
}

@article{stroblWhatFormalLanguages2024,
  title = {What {{Formal Languages Can Transformers Express}}? {{A Survey}}},
  shorttitle = {What {{Formal Languages Can Transformers Express}}?},
  author = {Strobl, Lena and Merrill, William and Weiss, Gail and Chiang, David and Angluin, Dana},
  year = 2024,
  month = may,
  journal = {Transactions of the Association for Computational Linguistics},
  volume = {12},
  pages = {543--561},
  issn = {2307-387X},
  doi = {10.1162/tacl_a_00663},
  urldate = {2026-09-25}
}

@inproceedings{
wang2023selfconsistencyimproveschainthought,
title={Self-Consistency Improves Chain of Thought Reasoning in Language Models},
author={Xuezhi Wang and Jason Wei and Dale Schuurmans and Quoc V Le and Ed H. Chi and Sharan Narang and Aakanksha Chowdhery and Denny Zhou},
booktitle={The Eleventh International Conference on Learning Representations },
year={2023}
}

@inproceedings{shinn2023reflexionlanguageagentsverbal, author = {Shinn, Noah and Cassano, Federico and Gopinath, Ashwin and Narasimhan, Karthik and Yao, Shunyu}, title = {Reflexion: language agents with verbal reinforcement learning}, year = {2023}, publisher = {Curran Associates Inc.}, address = {Red Hook, NY, USA}, booktitle = {Proceedings of the 37th International Conference on Neural Information Processing Systems}, articleno = {377}, numpages = {19}, location = {New Orleans, LA, USA}, series = {NIPS '23} }

@article{Besta_2024,
   title={Graph of Thoughts: Solving Elaborate Problems with Large Language Models},
   volume={38},
   ISSN={2159-5399},
   number={16},
   journal={Proceedings of the AAAI Conference on Artificial Intelligence},
   publisher={Association for the Advancement of Artificial Intelligence (AAAI)},
   author={Besta, Maciej and Blach, Nils and others},
   year={2024},
   month=mar, pages={17682–17690} }

@article{muennighoff2025s1simpletesttimescaling,
  title = {S1: {{Simple}} Test-Time Scaling},
  shorttitle = {S1},
  author = {Muennighoff, Niklas and Yang, Zitong and Shi, Weijia and Li, Xiang Lisa and {Fei-Fei}, Li and Hajishirzi, Hannaneh and Zettlemoyer, Luke and Liang, Percy and Cand{\`e}s, Emmanuel and Hashimoto, Tatsunori},
  year = 2025,
  month = mar,
  number = {arXiv:2501.19393},
  eprint = {2501.19393},
  primaryclass = {cs.CL},
  journal = {arXiv},
  doi = {10.48550/arXiv.2501.19393},
  urldate = {2026-09-26},
  archiveprefix = {arXiv}
}

@inproceedings{yao2023treethoughtsdeliberateproblem, author = {Yao, Shunyu and Yu, Dian and others}, title = {Tree of thoughts: deliberate problem solving with large language models}, year = {2023}, publisher = {Curran Associates Inc.}, address = {Red Hook, NY, USA}, booktitle = {Proceedings of the 37th International Conference on Neural Information Processing Systems}, articleno = {517}, numpages = {14}, location = {New Orleans, LA, USA}, series = {NIPS '23} }

@inproceedings{
li2025compressingchainofthoughtllmsstep,
title={Making Slow Thinking Faster: Compressing {LLM} Chain-of-Thought via Step Entropy},
author={Zeju Li and Jianyuan Zhong and Ziyang Zheng and Xiangyu Wen and Zhijian Xu and Yingying Cheng and Fan Zhang and Qiang Xu},
booktitle={The Fourteenth International Conference on Learning Representations},
year={2026}
}

@article{
sui2025stopoverthinkingsurveyefficient,
title={Stop Overthinking: A Survey on Efficient Reasoning for Large Language Models},
author={Yang Sui and Yu-Neng Chuang and others},
journal={Transactions on Machine Learning Research},
issn={2835-8856},
year={2025}
}

@inproceedings{
wu2025lessunderstandingchainofthoughtlength,
title={When More is Less: Understanding Chain-of-Thought Length in {LLM}s},
author={Yuyang Wu and Yifei Wang and Ziyu Ye and Tianqi Du and Stefanie Jegelka and Yisen Wang},
booktitle={The Fourteenth International Conference on Learning Representations},
year={2026}
}

@inproceedings{aghajohariMarkovianThinkerArchitectureAgnostic2025,
  title = {The {{Markovian Thinker}}: {{Architecture-Agnostic Linear Scaling}} of {{Reasoning}}},
  shorttitle = {The {{Markovian Thinker}}},
  booktitle = {The {{Fourteenth International Conference}} on {{Learning Representations}}},
  author = {Aghajohari, Milad and Chitsaz, Kamran and Kazemnejad, Amirhossein and Chandar, Sarath and Sordoni, Alessandro and Courville, Aaron and Reddy, Siva},
  year = 2025,
  month = oct,
  urldate = {2026-09-26},
  langid = {english}
}

@article{nguyenSequenceModelingDesign2024,
  title = {Sequence Modeling and Design from Molecular to Genome Scale with {{Evo}}},
  author = {Nguyen, Eric and Poli, Michael and Durrant, Matthew G. and Kang, Brian and Katrekar, Dhruva and Li, David B. and Bartie, Liam J. and Thomas, Armin W. and King, Samuel H. and Brixi, Garyk and Sullivan, Jeremy and Ng, Madelena Y. and Lewis, Ashley and Lou, Aaron and Ermon, Stefano and Baccus, Stephen A. and {Hernandez-Boussard}, Tina and R{\'e}, Christopher and Hsu, Patrick D. and Hie, Brian L.},
  year = 2024,
  month = nov,
  journal = {Science},
  volume = {386},
  number = {6723},
  pages = {eado9336},
  publisher = {American Association for the Advancement of Science},
  doi = {10.1126/science.ado9336},
  urldate = {2026-10-01}
}

@article{compeauWhyAreBruijn2011,
  title = {Why Are de {{Bruijn}} Graphs Useful for Genome Assembly?},
  author = {Compeau, Phillip E. C. and Pevzner, Pavel A. and Tesler, Glenn},
  year = 2011,
  month = nov,
  journal = {Nature biotechnology},
  volume = {29},
  number = {11},
  pages = {987--991},
  issn = {1087-0156},
  doi = {10.1038/nbt.2023},
  urldate = {2026-10-01},
  pmcid = {PMC5531759},
  pmid = {22068540}
}

@article{bankevichMultiplexBruijnGraphs2022,
  title = {Multiplex de {{Bruijn}} Graphs Enable Genome Assembly from Long, High-Fidelity Reads},
  author = {Bankevich, Anton and Bzikadze, Andrey V. and Kolmogorov, Mikhail and Antipov, Dmitry and Pevzner, Pavel A.},
  year = 2022,
  month = jul,
  journal = {Nature Biotechnology},
  volume = {40},
  number = {7},
  pages = {1075--1081},
  publisher = {Nature Publishing Group},
  issn = {1546-1696},
  doi = {10.1038/s41587-022-01220-6},
  urldate = {2026-10-01},
  copyright = {2022 The Author(s), under exclusive licence to Springer Nature America, Inc.},
  langid = {english}
}

@article{belkin2019reconciling,
  title = {Reconciling Modern Machine-Learning Practice and the Classical Bias--Variance Trade-Off},
  author = {Belkin, Mikhail and Hsu, Daniel and Ma, Siyuan and Mandal, Soumik},
  year = 2019,
  journal = {Proceedings of the National Academy of Sciences},
  volume = {116},
  number = {32},
  pages = {15849--15854}
}

@article{bousquet2002stability,
  title = {Stability and Generalization},
  author = {Bousquet, Olivier and Elisseeff, Andr{\'e}},
  year = 2002,
  journal = {Journal of machine learning research},
  volume = {2},
  number = {Mar},
  pages = {499--526}
}

@article{cover1967nearest,
  title = {Nearest Neighbor Pattern Classification},
  author = {Cover, Thomas and Hart, Peter},
  year = 1967,
  journal = {IEEE transactions on information theory},
  volume = {13},
  number = {1},
  pages = {21--27}
}

@article{minsky1969introduction,
  title = {An Introduction to Computational Geometry},
  author = {Minsky, Marvin and Papert, Seymour},
  year = 1969,
  journal = {Cambridge tiass., HIT},
  volume = {479},
  number = {480},
  pages = {104}
}

@article{Valiant1984-dx,
  title = {A Theory of the Learnable},
  author = {Valiant, L G},
  year = 1984,
  month = nov,
  journal = {Communications of the ACM},
  volume = {27},
  number = {11},
  pages = {1134--1142},
  address = {New York, NY, USA}
}

@book{vapnik1998statistical,
  title = {Statistical Learning Theory},
  author = {Vapnik, Vladimir},
  year = 1998
}

@inproceedings{xu2017information,
  title = {Information-Theoretic Analysis of Generalization Capability of Learning Algorithms},
  booktitle = {{{NIPS}}},
  author = {Xu, Aolin and Raginsky, Maxim},
  year = 2017
}

\appendix
\crefalias{section}{appendix}
\crefalias{subsection}{subappendix}
\crefalias{subsubsection}{subappendix}
\section{mathematical results about de bruijn graphs}
\label{app:debruijn_math}
In the following theorems and their proofs, $B(V,c)$ is a De Bruijn graph of size $V>4$, $c\geq 2$ and $D_{\pi}$ refers to the corresponding De Bruijn DAG with $S$ start nodes, $N$ answer nodes and topological sort $\pi$.

\subsection{$D_{\pi}$ is a dag with many start-to-answer paths}
\begin{lemma}
    $D_{\pi}$ is a DAG for any ordering of nodes $\pi$. 
\end{lemma}
\begin{proof}
    We will prove, by contradiction, that the generated graph has no cycles. Assume that $D_{\pi}$ contains a cycle. Then, it contains a sequence of nodes $(u_1, u_2,\dots, u_k)$ for some integer $k \geq 1$ such that there is a directed edge from node $u_i$ to $u_{i+1}$ for any $1 \leq i < k$, and there is a directed edge from $u_k$ to $u_1$. This means that $\pi(u_1) <\dots < \pi(u_k) < \pi(u_1) \implies \pi(u_1) < \pi(u_1)$ which is clearly an impossibility. By contradiction, the procedure generates a DAG.
\end{proof}

\begin{lemma}
    There exists a number of start nodes $S > 0$ and a permutation $\pi$ such that $D_{\pi}$ contains the maximal count of distinct paths from a start node to the $N$ answer nodes of any DAG that can be constructed from the edges in $B(V,c)$. 
\end{lemma}
\begin{proof}
    Let $D^*$ be a DAG that only contains edges from the De Bruijn graph $B(V,c)$, maximizes the number of distinct paths from its $S^*$ start nodes to the $N$ answer nodes and does not contain any vacuous edges or nodes that are not on a start-to-answer path. There exists a sort $\pi^*$ of the nodes in $B(V,c)$ such that the ordering in $\pi^*$ defines a topological sort over the nodes in $D^*$, the $S^*$ start nodes in $D^*$ come first in the ordering and the $N$ answer nodes last. Let $D_{\pi^*}$ be the DAG generated with the optimal number of start nodes $S^*$ and ordering of nodes $\pi^*$. The procedure to create $D_{\pi^*}$ deletes any edge from node $u$ to $v$ such that $\pi^*(u) \geq \pi^*(v)$. Since $\pi^*$ defines a topological sort of the nodes in $D^*$, no edge in $D^*$ will be deleted by this procedure. Furthermore, both $D_{\pi^*}$ and $D^*$ have the same set of start and answer nodes, and $D^*$ does not contain any edges that are not in the De Bruijn graph $B(V,c)$. This means that the generated graph $D_{\pi^*}$ contains all the edges in $D^*$, so $D^* \subseteq D_{\pi^*}$. Therefore the number of distinct paths from start to answer nodes in $D_{\pi^*}$ is greater than or equal to the number of paths of that kind in $D^*$. Since $D^*$ is a maximizer of the number of paths, $D_{\pi^*}$ is as well.
\end{proof}

 These lemmas establish that the De Bruijn DAG $D_{\pi}$ captures the kinds of reasoning structures we are interested in, i.e., those with a large number of finite-length reasoning traces.

\subsection{proofs of theorems and corollaries}

\begin{proof}[Proof of \cref{theorem1}]
    Define the alphabet of size $V$ as $A = \{0,\dots, V-1\}$. Each node $u$ in $D_{\pi}$ can be written as $u=a_1a_2\dots a_c$ for all $a_i\in A$. We will also use shorthand such as $0^k$ or $(V-1)^k$ to refer to length $k$ strings of the repeated character $0$ or $(V-1)$ respectively. While this shorthand could be mistaken as exponents, its meaning will be stated or clear from the context. Define the digit-sum of this node $u$ as $$\sigma(u) = \sum_{i=1}^c a_i.$$And let $\pi$ be an ordering by this digit sum such that for any two nodes $u_1,u_2$ in $D_{\pi}$, $$\pi(u_1) < \pi(u_2)\implies \sigma(u_1)\leq \sigma(u_2).$$
    Tie-breaking between two nodes with equal digit-sum is done at random.
    
    First, we prove the bound on the total number of start-to-answer paths $P_{\text{tot}}$. 
	Consider the original De Bruijn graph $B(V,c)$ from which the DAG $D_{\pi}$ was derived. There is an edge $u\to v$ in $B(V,c)$ if and only if we can write  $u = a_1a_2\dots a_c$ and $v=a_2a_3\dots b$ for all $a_i, b \in A$. This edge remains in $D_{\pi}$ if $\sigma(v) > \sigma (u)$ so $b > a_1$. Taking this one step further, consider the directed path with $l$ nodes, $u_1\to u_2\to\dots\to u_l$ in $B(V,c)$. This path is generated from the string $a_1a_2\dots a_{c+l-1}$ where each $u_j = a_ja_{j+1}\dots a_{c+j-1}$. The path is guaranteed to exist in $D_{\pi}$ if (i) $u_1$ is a start and $u_l$ is an answer, (ii) no intermediate node $u_i$ for $1 < i < l$ is a start or an answer, (iii) and every edge in the path points to a node with a greater digit-sum than the last. The third condition is equivalently rewritten as $$a_{c+j} > a_j$$ for any $1\leq j \leq l-1$. This observation will be useful later in the proof. 
	Recall the start nodes are the first $S$ nodes in $\pi$ and the answer nodes are the last $N$ nodes in $\pi$.  Because the number of start and answer nodes are bounded by ${c+m-1\choose m-1}$, each start node has digit-sum $< m$ and each answer node has digit-sum $> c(V-1)-m$. This is because there are exactly ${c+m-1\choose m-1}$ possible ways to create start or answer nodes with their respective allowed digit-sums (by the stars and bars method). In addition, the node with the smallest digit-sum $0^c$ is guaranteed to be a start and the node $(V-1)^c$ with the largest digit-sum is guaranteed to be an answer node. Therefore, the number of distinct paths from $0^c$ to $(V-1)^c$ in $D_{\pi}$ is a lower bound on the total number of start-to-answer paths in $D_{\pi}$. We would like to estimate this value.
	A generic path between these two nodes $0^c \to u_1\to\dots\to u_{tc+c-1}\to (V-1)^c$ with $(t+1)c-1$ intermediate nodes can be generated from the string $0^ca_1a_2\dots a_{tc}(V-1)^c$. This is a valid path in $D_{\pi}$ if
	(i) $a_i \notin \{0,1,\dots, m-1\} \cup \{V-1,V-2,\dots, (V-1)-m+1\}$ which guarantees that no intermediate node in the path is a start or an answer,
	(ii) $a_{j+c} > a_j$ for any $1 \leq j \leq c(t-1)$.
	These conditions are repeated from our observation above. Satisfying them requires that $0 \leq t \leq (V-2m)$. The number of such paths is just the number of monotonically increasing sequences of length $t$, raised to the $c$ and summed over all $t$ $$\sum_{t=0}^{V-2m}\left[{V-2m \choose t}\right]^c.$$
    
	Second, we prove the bound on $P_{\text{min}}$ for the same digit-sum ordering $\pi$. First note that the minimum number of paths required to cover every edge in $D_{\pi}$, denoted by $P_{\min}$, is at most the number of edges in $D_{\pi}$. The following argument proves this fact. Consider the set $\mathcal P_{\text{min}}$ of paths in $D_{\pi}$ that cover all the edges in $D_{\pi}$ and is of size $P_{\text{min}}$.  Each path $p\in \mathcal P_{\text{min}}$ must contain at least one edge in $D_{\pi}$ that is not in any other path in $\mathcal P_{\text{min}}$. Otherwise the set $\mathcal P_{\text{min}}\setminus \{p\}$ has size $P_{\text{min}}-1$ and still covers every edge in $D_{\pi}$. This is in contradiction with the definition that $P_{\text{min}}$ is the minimum number of paths required to cover all the edges in $D_{\pi}$. Since every path $p\in\mathcal P_{\text{min}}$ is associated with at least one unique edge, $P_{\text{min}}$ can be at most the number of edges in $D_{\pi}$. We would like to upper bound the number of edges in $D_{\pi}$ which will give us an upper bound on $P_{\text{min}}$. 
	For every pair of nodes $u=a_1\dots a_c$ and $v=b_1\dots b_c$ for $a_i,b_i\in A$, an edge exists from $u\to v$ only if $b_c \geq a_1$ and $a_{i+1}=b_i$ for all $1 \leq i \leq c-1$. There are $$\frac{V^c(V+1)}{2}$$ordered pairs of nodes satisfying these properties which is an upper bound on $P_{\text{min}}$.

    By taking the largest term in the bound on $P_{\text{tot}}$ in \cref{theorem1}, we know that $$P_{\text{tot}} \geq \left[\binom{V-2m}{\lfloor (V-2m)/2\rfloor}\right]^c.$$ In the asymptotic limit $V\to \infty$ we apply Stirling's approximation and get $$P_{\text{tot}} \geq \left(\frac{2}{\pi}\right)^{c/2}\frac{2^{c(V-2m)}}{(V-2m)^{c/2}}.$$ In this limit, the ratios $|E|/P_{\text{tot}}$ and $P_{\text{min}}/P_{\text{tot}}$ decay to $0$ exponentially quickly as $V$ increases.
	Next consider the same two ratios $$\frac{P_{\text{min}}}{P_{\text{tot}}} \leq \frac{|E|}{P_{\text{tot}}} \leq \left[\frac{V+1}{2}\left(\frac{V}{\binom{V-2m}{\lfloor (V-2m)/2\rfloor}}\right)^c\right]$$ in the limit that $c\to \infty$. For sufficiently large $V$, the fraction $V/{V-2m \choose \lfloor (V-2m)/2\rfloor} < 1$. Thus, taking the limit $c\to\infty$ yields $0$ for both ratios.
\end{proof}

\begin{proof}[Proof of \cref{theorem2}]
    Let $\pi$ be the digit-sum ordering as described in the proof of Theorem 1. If two nodes $u=a_1a_2\dots a_c$ and $v=b_1b_2\dots b_c$ have the same digit-sum, we will break ties by lexicographically comparing their respective reversed strings $\bar u = a_ca_{c-1}\dots a_1$ and $\bar v = b_cb_{c-1}\dots b_1$. Specifically, $\bar v < \bar u$ in the lexicographic comparison if, for $i^* = \max\{i: a_i\neq b_i\}$, $b_{i^*} < a_{i^*}$. If $\bar v < \bar u$ lexicographically and both nodes have the same digit-sum, then $\pi(v) > \pi(u)$. 
	
    First, we bound the longest start-to-answer path. We can construct a path $u_1\to \dots \to u_{L}$ based on the string $$w=w_1\dots w_{c(V-2m+2)} = 0^cm^c(m+1)^c\dots (V-1-m)^c(V-1)^c$$where each node $u_j = w_jw_{j+1}\dots w_{j+c-1}$ is a $c$-length slice of the string. The digit-sum is strictly increasing along this path, $\sigma(u_{j+1}) - \sigma(u_j) > 0$, the first node $u_1=0^c$ is guaranteed to be a start and the last node $u_L=(V-1)^c$ is an answer. In addition, recall from the proof of Theorem 1 that every start node $s$ has digit-sum $\sigma(s)< m$ and every answer node $s'$ has digit-sum $\sigma(s') > c(V-1)-m$, so no intermediate node $u_j$ for $1<j <L$ is a start or an answer node. Therefore, the path $u_1\to\dots\to u_L$ is contained in $D_{\pi}$ where $L=c(V-2m+1)+1$. This path has length $c(V-2m+1)$, proving the first bound in the theorem. 
    
	Second, we prove the bound on $L_{\text{cover}}$, i.e., the length of paths required to cover every edge in $D_{\pi}$. Let $u\to v$ be an arbitrary edge in $D_{\pi}$. We will show that this edge is contained in a path of length less than or equal to $2c+1$. We call an edge $u\to v$  of $B(V,c)$ \textit{admissible} if $u$ is not an answer node, $v$ is not a start node, and $\pi(u) < \pi(v)$. Every edge of $D_{\pi}$ is admissible, but not every admissible edge in $B(V,c)$ is in $D_{\pi}$ because any admissible edges that are not on a start-to-answer path in $D_{\pi}$ are pruned.
    
	We start by proving that there exists a path from a start node to $u = a_1a_2\dots a_c$ that has length less than or equal to $c$. Consider a sequence of $c+1$ nodes $(u_0,u_1,\dots,u_c)$ such that $u_0=0^c$, $u_c=u$ and each intermediate $u_j = 0^{c-j}a_1a_2\dots a_{j}$. Let $i_0$ be the index of the latest start node in the sequence. We will prove that the path $u_{i_0} \to u_{i_0+1} \to \dots \to u_c$ is a valid path in $D_{\pi}$ of length $c-i_0$. If $u$ is a start node, then the path is of length $0$ and is trivially in $D_{\pi}$. In addition, no node in the path is an answer node since $u$ has an outgoing edge to $v$ and the intermediate nodes have a digit-sum of at most $(c-1)(V-1)$ which is smaller than the minimum answer node digit-sum of $c(V-1)-m+1$. Consider a transition $u_{j-1}\to u_{j}$ in the path for $j> i_0$. There are two cases.
    \begin{enumerate}
        \item If $a_j > 0$, then $\sigma(u_{j}) > \sigma(u_{j-1})$ so $\pi(u_{j}) > \pi(u_{j-1})$ and the edge $u_{j-1}\to u_j$ is admissible.
        \item If $a_j = 0$, we can compare $u_j$ with $u_{j-1}$ using the reverse lexicographic sort. Let $i^* = \max\{i : i < j \text{ and }a_i > 0\}$ which is well-defined for any $u_j$ in the path $u_{i_0} \to u_{i_0+1} \to \dots \to u_c$ since $u_j\neq 0^c$ for any $j > i_0$. The character $a_{i^*}$ is at position $c-j+i^*+1$ in node $u_{j-1}$ and it is at the earlier position $c-j+i^*$ in node $u_j$ with every subsequent character being $0$, which is smaller than $a_{i^*}$. This means $\bar u_j < \bar u_{j-1}$ in the reversed lexicographic sort which we are using for tie-breaking, so $\pi(u_{j}) > \pi(u_{j-1})$ and, once again, the edge $u_{j-1}\to u_{j}$ is admissible. 
    \end{enumerate}
    
	Therefore, every edge of $u_{i_0} \to u_{i_0+1} \to \dots \to u$ is admissible. This means there is a walk of admissible edges from a start node to an arbitrary node $u$, of length $\leq c$.
    
	Next, we prove that there exists a path from an arbitrary node $v=b_1b_2\dots b_c$ in $D_{\pi}$ to an answer node that has length less than or equal to $c$. There are two cases.
    \begin{enumerate}
        \item Each $b_j < (V-1)$. Consider a sequence of $c+1$ nodes $(v_0, v_1, \dots, v_c)$ such that $v_c = (V-1)^c$ is the largest possible node, $v_0=v$ and each intermediate $v_j = b_{j+1}b_{j+2}\dots b_{c}(V-1)^j$. Let $k_0$ be the index of the earliest answer node in the sequence. We will prove that the path $v_0\to v_1\to\dots \to v_{k_0}$ is a valid path in $D_{\pi}$. If $v$ is an answer node, then the path has length $0$ and is trivially in $D_{\pi}$. In addition, no node in the path is a start node since $v$ has an incoming edge from $u$ and the intermediate nodes have a digit-sum of at least $V-1$ which is larger than the maximum start node digit-sum of $m-1$. Consider a transition $v_{j-1} \to v_j$ in the path. Since $b_{j} < V-1$, $\sigma(v_{j}) > \sigma(v_{j-1})$ so the edge $v_{j-1} \to v_j$ is admissible. Therefore, the path $v_0\to v_1\to\dots \to v_{k_0}$ is a walk of admissible edges of length $\leq c$.
        \item There exists a $j$ such that $b_{j} = V-1$. Let $i_1 = \min\{i : b_i=(V-1)\}$. Consider a path $v_0\to v_1\to\dots\to v_k$ for $v_0=v$ and to an answer node $v_k$. This path can be generated by a string $b_1b_2\dots b_{c+k}$ where each $v_j = b_{j+1}\dots b_{j+c}$. We will prove by contradiction that $k \leq i_1$, otherwise the path cannot exist in $D_{\pi}$. Consider $k > i_1$ and assume the path exists in $D_{\pi}$. There is a transition $v_{i_1-1}\to v_{i_1}$ in the path. This transition can be written in characters as $(V-1) b_{i_1+1}\dots b_{c+i_1-1}\to b_{i_1+1}\dots b_{c+i_1}$. Since the transition is represented by an edge in $D_{\pi}$, it must be true that $\sigma(v_{i_1-1}) \leq \sigma(v_{i_1})\implies (V-1)\leq b_{c+i_1}\implies b_{c+i_1}=(V-1) \implies \sigma(v_{i_1-1}) = \sigma(v_{i_1})$, resulting in a tie for the digit-sum. Comparing the reverse lexicographic sort yields $\bar v_{i_1} > \bar v_{i_1-1} \implies \pi(v_{i_1}) < \pi(v_{i_1-1})$  provided $v_{i_1} \neq (V-1)^c$ is not an answer (true by assumption). This is a result of a direct comparison of $v_{i_1-1}=(V-1)d_1\dots d_{c-1}$ to $v_{i_1}=d_1\dots d_{c-1}(V-1)$ where at least one of the $d_j < V-1$. This means there cannot be an edge from $v_{i_1-1}\to v_{i_1}$ contradicting our assertion that the path $v_0\to v_1\to\dots\to v_k$ exists in $D_{\pi}$. Therefore it must be true that $k\leq i_1\leq c$. It follows that, if there exists a path in $D_{\pi}$ from $v$ to an answer node (true by assumption that the edge $u\to v$ exists in $D_{\pi}$), its length is upper-bounded by $c$.
    \end{enumerate}
		 
	If the edge $u\to v$ exists in $D_{\pi}$, then by the two arguments above there is a walk of admissible edges from a start node to $u$ of length $\leq c$, and a walk of admissible edges from $v$ to an answer node of length $\leq c$. Concatenating these with the edge $u\to v$ gives a walk $W$ from a start node to an answer node in which every edge is admissible, so $\pi$ is strictly increasing along $W$. Hence no node of $W$ repeats and $W$ is a path. Since $W$ runs from a start node to an answer node, all the edges of $W$ are in $D_{\pi}$, so $W\subseteq D_{\pi}$. The length of $W$ is at most $2c+1$ and it contains any edge $u\to v$. 

    A straightforward evaluation of the limit $\lim_{c\to \infty} \frac{2c+1}{c(V-2m+1)}$ shows that the ratio goes to zero asymptotically.
\end{proof}




\section{details for the synthetic experiments in \cref{fig:synthetic_exp} and \cref{fig:remapped}}
\label{app:synth_exp}
We trained small ($\sim$0.8M parameter) decoder-only Transformer models to auto-regressively predict token sequences. Each sequence corresponds to a start-to-answer path on a De Bruijn DAG. These DAGs were created by the same procedure used when defining $D_{\pi}$. We chose to scale the number of start nodes $S$ and answer nodes $N$ with the size of the De Bruijn graph, but our results were robust to ablations across both values. The topological sorts $\pi$ were chosen uniformly at random for five replicate experiments per choice of $V$ and $c$. For each of these DAGs, we created three corresponding remapped DAGs where token values were permuted at random once every 4, 8 or 16 tokens.

Our experiments were on DAGs of the following conditions.
\begin{table}[h]
    \centering
    \begin{tabular}{cccc}
        \toprule
        $V$ & $c$ & $S$ & $N$ \\
        \midrule
        8  & 2 & 5  & 8  \\
        10 & 2 & 6  & 10 \\
        12 & 2 & 8  & 12 \\
        \midrule
        8  & 3 & 10 & 20 \\
        10 & 3 & 12 & 24 \\
        11 & 3 & 14 & 28 \\
        12 & 3 & 15 & 30 \\
        \midrule
        8  & 4 & 12 & 24 \\
        9  & 4 & 14 & 28 \\
        10 & 4 & 16 & 32 \\
        \bottomrule
    \end{tabular}
    \caption{The parameters for the De Bruijn DAGs in the experiment.}
    \label{tab:conditions}
\end{table}

For each DAG, we trained a decoder-only pre-norm Transformer model to predict its start-to-answer paths from scratch. The model vocabulary size was equal to the vocabulary size of the DAG with an additional special end of sequence token. It used a linear embedding layer with embedding dimension 128 followed by 4 subsequent Transformer Blocks and then a root-mean-squared norm (RMSNorm) before the output linear layer that predicts logits. The input embedding, output linear layer and RMSNorm weights were all trainable. Each Transformer block consisted of, first, an RMSNorm followed by a multi-head attention with Rotary Positional Embeddings (RoPE) and a residual addition between the output of the attention module and the original input to the Transformer block. The attention module used 4 heads of dimension 32 each. The RoPE used wavelengths of $10000^{(2k)/d_k}$ where $d_k=32$ is the dimension of the key and query vectors and $k$ is half the feature index, rounded down to the nearest integer. Following the first residual connection, there was a second RMSNorm and then a Swish-Gated Linear Unit (SwiGLU) with a feed-forward dimension 341. Finally, there was a second residual connection which added the output of the SWiGLU block to the output of the first residual connection after the attention computation. 

To evaluate the models, we drew 8000 samples uniformly at random over the set of start-to-answer paths in the DAG and used them to compute the path coverage and P(invalid tokens) metrics. 

We trained the models to predict the next token in the sequences using cross-entropy loss. We used the AdamW optimizer with learning rate 0.001, weight decay 0.01, betas (0.9, 0.98) and batch size 512. For~\cref{fig:synthetic_exp} (left), we trained the model for one epoch and evaluated both Path coverage and P(invalid tokens) at regular geometrically-spaced intervals between 316 samples and 316,228 samples with 22 intermediate points (each point was roughly 1.4$\times$ the previous). Training was stopped at the intermediate point where both criteria were met: path coverage $>$ 0.95 and P(invalid tokens) $<$ 0.05. Then, we used linear interpolation to estimate the exact number of samples at which these thresholds were crossed. The geometric intervals were chosen to be sufficiently tight that increasing the number of intermediate points did not have any significant impact on the results. 

To create~\cref{fig:synthetic_exp} (right), the model was trained for 12 epochs on 64,000 training samples that were drawn uniformly at random from the set of start-to-answer paths given that they had length smaller than the training cutoff $L_{\text{train}}$. We evaluated path coverage and P(invalid tokens) on every value of $L_{\text{train}}$ until both criteria were met and additionally plotted the longest covered path for that training length. 

\section{details of the story task in \cref{fig:tom_results} and \cref{fig:storytask}}
\label{app:story}
\subsection{example problems and reasoning}
An example problem is
\begin{framed}
\footnotesize
\noindent Story: Nicholas entered the hostel dorm. Julia entered the hostel dorm.
Nicholas moved the harmonica to the canvas bag, which is also located in
the hostel dorm. Nicholas left the hostel dorm. Julia moved the harmonica
to the leather pouch, which is also located in the hostel dorm.

\noindent Question: In which container will Nicholas search for the harmonica?
\end{framed}

The stories used in training and evaluation are much more substantial than this example, with 20 actions and multiple objects to track. Examples of each kind of reasoning are below.
\begin{enumerate}
\item \textbf{Direct} prediction did not use any reasoning.
\begin{framed}
\footnotesize
\noindent Answer: canvas bag
\end{framed}

\item \textbf{Key-steps} only restated the parts of the story that changed the final answer.
\begin{framed}
\footnotesize
\noindent Tracking the answer to the question through the story.
Nicholas moved the harmonica to the canvas bag, which is
also located in the hostel dorm. $\to$ answer now: canvas bag.
Answer: canvas bag
\end{framed}
\item \textbf{State-tracking} repeated every line of the story followed by a state containing information about the characters and their beliefs that were relevant to the question.
\begin{framed}
\footnotesize
\setlength{\parindent}{0em}
Nicholas entered the hostel dorm.\\
State: Nicholas is in the hostel dorm. Beliefs: none yet.\\
Julia entered the hostel dorm.\\
State: Nicholas is in the hostel dorm. Beliefs: none yet.\\
Nicholas moved the harmonica to the canvas bag, which is
also located in the hostel dorm.\\
State: Nicholas is in the hostel dorm; the harmonica is
in the canvas bag in the hostel dorm. Beliefs: Nicholas
believes the harmonica is in the canvas bag in the hostel dorm.
Nicholas left the hostel dorm.\\
State: Nicholas has left the hostel dorm; the harmonica is in
the canvas bag in the hostel dorm. Beliefs: Nicholas believes
the harmonica is in the canvas bag in the hostel dorm.\\
Julia moved the harmonica to the leather pouch, which is also
located in the hostel dorm.\\
State: Nicholas has left the hostel dorm; the harmonica is in
the leather pouch in the hostel dorm (it started in the canvas bag).
Beliefs: Nicholas believes the harmonica is in
the canvas bag in the hostel dorm.\\
Answer: canvas bag
\end{framed}
In this example, the model emits the state after every step. We also studied conditions with larger state-emission intervals.
\item \textbf{Narration} repeated every line of the story, followed by a restatement of the local changes made by that line.
\begin{framed}
\footnotesize
\setlength{\parindent}{0em}
Nicholas entered the hostel dorm.\\
Note: Nicholas is now in the hostel dorm; present there: Nicholas.
Julia entered the hostel dorm.\\
Note: Julia is now in the hostel dorm; present there: Julia, Nicholas.\\
Nicholas moved the harmonica to the canvas bag,
which is also located in the hostel dorm.\\
Note: Nicholas put the harmonica into the canvas bag in the hostel dorm;
present there: Julia, Nicholas.\\
Nicholas left the hostel dorm.\\
Note: Nicholas left the hostel dorm; still there: Julia.\\
Julia moved the harmonica to the leather pouch, which is also located in
the hostel dorm.\\
Note: Julia put the harmonica into the leather pouch in the hostel dorm;
present there: Julia.\\
Answer: canvas bag
\end{framed}

\end{enumerate}
\subsection{data generation}
A story is a program over a small action language. People enter and leave rooms, move objects into containers or other rooms, tell one another things privately or out loud, and may secretly witness an action or be distracted and miss it. A state tracker maintains, after each action,
\begin{enumerate}
    \item The true location of each object (its container and room), each person's room, and each container's room. 
    \item The beliefs of each person about each object and container's location, as well as whether or not they have knowledge about a certain topic. These are first-order beliefs.
    \item The beliefs of each person about what each other character believes or the topics they know about. These are second-order beliefs. 
\end{enumerate}
Each story is generated by repeatedly updating the state tracker according to the following process. First, the entities in the story are sampled from a curated list. Four person names are sampled from a list of forty, three rooms are sampled from sixteen possibilities, four objects are sampled from twenty possibilities, three containers are sampled from twelve possibilities, and four topics are sampled from a list of 12. All actions that will be generated for the story are confined to using these sampled values. Each story starts with a deterministic opening. Each sampled person enters a uniformly random room (four actions). Each object is moved by a person into a container and the container's location is stated (four actions). After the initial setup, a sequence of twelve additional actions are randomly sampled to construct a storyline. Actions which are invalid because, e.g., an object is already in a container that the action meant to move it to, are resampled. The possible actions and the frequency at which they are chosen are described below. Finally, four questions are used per story, which are generated based on the final state. The questions where substituting different people would result in different answers are prioritized, since they require tracing each person's belief to answer. The procedure for generating questions is outlined below. Unlike~\citet{sclarExploreTheoryMind2024}, we do not use language models during data generation. 

\paragraph{Actions and their frequencies.}
At each step of the sampled body, one of seven action types is drawn
independently from the fixed distribution in Table~\ref{tab:actions}, and its
arguments (which person, object, container, room, or topic) are then chosen
uniformly at random among the sampled entities. A proposal that violates the
tracker's preconditions---for example, moving an object into a container when
the acting person is not in the same room as the object, or telling someone a
fact the speaker believes they already know---is rejected and the step is
resampled (up to a fixed budget). Roughly one third of proposed actions are communicative, e.g., `Alice tells Bob the potato is in the living room'. Private communication may reach a recipient in any room, whereas a broadcast communication reaches those in the speaker's room. 

\begin{table}[h]
\centering
\small
\begin{tabular}{@{}lp{0.60\linewidth}r@{}}
\toprule
Action & Templated sentence & Prob. \\
\midrule
Enter room & \{P\} entered the \{R\}. & 0.170 \\[2pt]
Leave room & \{P\} left the \{R\}. & 0.085 \\[2pt]
Move to container & \{P\} moved the \{O\} to the \{C\}, which is also located in the \{R\}. & 0.255 \\[2pt]
Move to room & \{P\} moved the \{O\} to the \{R\}, leaving the \{C\} in its original location. & 0.170 \\[2pt]
Tell privately (topic) & \{P1\} told privately to \{P2\} about the \{T\}. & 0.092 \\[2pt]
Tell privately (location) & \{P1\} told privately to \{P2\} that the \{O\} is in the \{C\}. & 0.137 \\[2pt]
Tell out loud & \{P\} told out loud about the \{T\}. & 0.092 \\
\bottomrule
\end{tabular}
\caption{The seven action types, the exact sentence template each emits, and the
probability with which each is proposed at every step. Placeholders:
\{P\}, \{P1\}, \{P2\} are persons, \{O\} an
object, \{C\} a container, \{R\} a room, and \{T\} a
topic. Probabilities follow from unnormalized base weights (2,1,3,2,2,3,2),
with the three communication rows scaled by $g_c=0.35$ and the four physical
rows by $1-g_c=0.65$ before normalizing. A witnessed action additionally emits ``While this action was happening, \{P\} witnessed this action in secret (and only this action).''}
\label{tab:actions}
\end{table}
Because invalid proposals are resampled, the realized mix of accepted
actions differs from the proposal distribution. For example, moving an object into a container is proposed most often but is valid only 29\% of the time (the
randomly chosen object is usually not in the actor's room), whereas moving an
object to another room is pre-filtered to valid choices and is accepted almost
always. Averaged over generated stories, the kept actions are approximately
25\% move-to-room, 16\% enter, 14\% leave, 12\% move-to-container, and
11\% each for the three communication types. For 14.3\% of private tellings, another person in the room secretly witnessed the event. In 2.8\% of cases, the recipient gets distracted and misses what they were told. Similarly, when a person places an object in a container, someone outside the room will secretly witness the event 19.9\% of the time, or someone inside the room will not notice the event 11.1\% of the time. 

\paragraph{Question generation.}
Once a story is complete, questions are enumerated deterministically from the
final belief state. The generator iterates over every first-order belief (a person, object pair that refers to the person's belief about the location of the object) and every second-order belief (a triple of two people and an object, referring to where one person thinks another person believes the object is), and instantiates a question for each. The answer to a belief question may differ from
the true location of the object. There are five kinds of questions:
\begin{framed}
\footnotesize
\setlength{\parindent}{0em}
  In which container will \{P\} search for the \{object\}?\\
  In which container does \{A\} think that \{B\} will search for the \{object\}?\\
  In which room does \{A\} think that \{B\} will search for the \{object\}?\\
  Does \{P\} know about \{topic\}? Answer yes or no.\\
  What does \{A\} think about \{B\}'s belief on \{topic\}? (knows about it / does not know about it)
\end{framed}
The first question is a \emph{container-location} question about a first-order belief, while the second is the same kind about a second-order belief. The third question is a \emph{room-location} question about a second-order belief. There are no room-location questions about a first-order belief. The final two questions are \emph{knowledge} questions about first- and second-order beliefs respectively. These categories are taken directly from~\citet{sclarExploreTheoryMind2024}. Knowledge questions have only two answers, while location questions are answered with the room or container name. A question is labeled as \emph{false-belief} when the queried belief disagrees with the truth, and \emph{interesting} when the answer depends on who is asked. 

For each story, we collect all questions, order the interesting ones first, and retain only four to use in the train or test dataset. Note that both datasets are split by stories, so a story in the test dataset will never overlap with one from the train dataset. The full breakdown of the generated problems by type is given in Table~\ref{tab:qtypes}. Overall $29.8\%$ of questions are false-belief questions, in which the queried belief disagrees with the truth.

\begin{table}[h]
\centering
\small
\begin{tabular}{lrr}
\toprule
Question type                  & Share & False-belief share \\
\midrule
Knowledge, second-order        & 55.9\% & 34.9\% \\
Knowledge, first-order         & 18.7\% & 0.0\% \\
Room location, second-order    & 11.8\% & 34.3\% \\
Container location, second-order & 8.4\% & 53.7\% \\
Container location, first-order  & 5.1\% & 33.3\% \\
\midrule
All                            & 100.0\% & 29.8\% \\
\bottomrule
\end{tabular}
\caption{Distribution of the $14{,}261$ generated questions by type, with the
fraction of each type that is a false-belief question (the queried belief
differs from the truth). Second-order questions constitute $76.2\%$
of the corpus.}
\label{tab:qtypes}
\end{table}

\subsection{model and training}
The model and training hyperparameters match the equation-solving task in~\cref{app:interp}. Training and evaluation were repeated across three seeds. Each of the three seeds induces an independent random split of the 479-story corpus into disjoint training and test sets. We used 375 stories for training and 63 held-out stories
for evaluation, with four questions drawn from any single story.
\subsection{evaluation}
Models were given the following system prompt
\begin{framed}
\footnotesize
\setlength{\parindent}{0em}
You read a short story and answer a question about it. You may reason first. Whatever else you write, the last line of your reply must be the answer field: "Answer: $<$answer$>$". Answer a yes/no question with yes or no, a question that offers two phrases in brackets with one of those phrases, and any other question with the name of the container, room or object.
\end{framed}

Every reply must end with an answer field, \texttt{Answer: <answer>} (see the
system prompt above). We take the text following the last such field and compare
it to the ground-truth answer by exact match after some light normalization:
lower-casing, removing markdown and brackets, dropping a leading article or
preposition (e.g., `the', `in'), stripping trailing punctuation, and canonicalizing yes/no. A reply that omits the answer field is graded on its last line, and generally scored incorrect. Answers take one of four forms:
\texttt{yes}/\texttt{no} for first-order knowledge questions,
\texttt{knows about it}/\texttt{does not know about it} for second-order
knowledge questions, a container name for container-location questions, or a
room name for second-order room-location questions. We report final-answer
accuracy on the $252$ held-out questions, averaged over the three seeds
($\pm$ standard error).

\section{details of the equation-solving task in \cref{fig:eqtask}, \cref{fig:math_eval_acc} and \cref{fig:math_eval_trends}}
\label{app:interp}
\subsection{data generation}

The training and test data consisted of a roughly even split of the following five kinds of tasks.
\begin{enumerate}
    \item \textbf{Linear}: Can be solved by a sequence of \textit{mostly} scaling (multiplication or division by an integer on both sides) and shifting (addition or subtraction by an integer on both sides) operations. 
    \begin{enumerate}
\item An example problem is:
\begin{framed}
\footnotesize
\setlength{\parindent}{0em}
Solve for x: -4*(2*x + 13) + 10 = -122\\
Output your answer after \#\#\#\#.
\end{framed}
\item The emission interval $k=\infty$ solution:
\begin{framed}
\footnotesize
\setlength{\parindent}{0em}
subtract 10 from both sides. divide both sides by -4. subtract 13 from both sides. divide both sides by 2. \#\#\#\# 10.
\end{framed}
\item The emission interval $k=1$ solution:
\begin{framed}
\footnotesize
\setlength{\parindent}{0em}
subtract 10 from both sides. Current equation is -4*(2*x + 13) = -132. divide both sides by -4. Current equation is 2*x + 13 = 33. subtract 13 from both sides. Current equation is 2*x = 20. divide both sides by 2. Current equation is x = 10. \#\#\#\# 10.
\end{framed}
\end{enumerate}

\item \textbf{Linear word problem}: A linear problem that is written in the form of a word problem. All word problems have the format in the example.
\begin{enumerate}
\item An example problem is:
\begin{framed}
\footnotesize
\setlength{\parindent}{0em}
I think of a number. I multiply it by -3, then subtract 13, then multiply it by -2. The result is 56. What is my number?\\
Output your answer after \#\#\#\#.
\end{framed}

\item The emission interval $k=\infty$ solution:
\begin{framed}
\footnotesize
\setlength{\parindent}{0em}
write the equation -2*(-3*x - 13) = 56. divide both sides by -2. add 13 to both sides. divide both sides by -3. \#\#\#\# 5.
\end{framed}

\item The emission interval $k=1$ solution:
\begin{framed}
\footnotesize
\setlength{\parindent}{0em}
write the equation -2*(-3*x - 13) = 56. Current equation is -2*(-3*x - 13) = 56. divide both sides by -2. Current equation is -3*x - 13 = -28. add 13 to both sides. Current equation is -3*x = -15. divide both sides by -3. Current equation is x = 5. \#\#\#\# 5.
\end{framed}
\end{enumerate}

\item \textbf{Quadratic}: A quadratic equation that is solved using a sequence of \textit{mostly} scaling and shifting operations, followed by factorization. Up to two solutions.
\begin{enumerate}

\item An example problem is:
\begin{framed}
\footnotesize
\setlength{\parindent}{0em}
Solve for x: 3*(3*(x**2 + 9*x + 8) + 10) = 30\\
Output your answer after \#\#\#\#.
\end{framed}

\item The emission interval $k=\infty$ solution:
\begin{framed}
\footnotesize
\setlength{\parindent}{0em}
divide both sides by 3. subtract 10 from both sides. divide both sides by 3. factor the left hand side. \#\#\#\# -8, -1.
\end{framed}

\item The emission interval $k=1$ solution:
\begin{framed}
\footnotesize
\setlength{\parindent}{0em}
divide both sides by 3. Current equation is 3*(x**2 + 9*x + 8) + 10 = 10. subtract 10 from both sides. Current equation is 3*(x**2 + 9*x + 8) = 0. divide both sides by 3. Current equation is x**2 + 9*x + 8 = 0. factor the left hand side. Current equation is (x + 1)*(x + 8) = 0. \#\#\#\# -8, -1.
\end{framed}
\end{enumerate}

\item \textbf{Cubic}: A cubic equation that is solved using a sequence of \textit{mostly} scaling and shifting operations, followed by factorization. Up to three solutions.
\begin{enumerate}
\raggedright
\item An example problem is:
\begin{framed}
\footnotesize
\setlength{\parindent}{0em}
Solve for x: -23 = 2*(x**3 - x**2 - 2*x - 6) - 11\\
Output your answer after \#\#\#\#.
\end{framed}

\item The emission interval $k=\infty$ solution:
\begin{framed}
\footnotesize
\setlength{\parindent}{0em}
add 11 to both sides. divide both sides by 2. add 6 to both sides. factor the left hand side. \#\#\#\# -1, 0, 2.
\end{framed}

\item The emission interval $k=1$ solution:
\begin{framed}
\footnotesize
\setlength{\parindent}{0em}
add 11 to both sides. Current equation is -12 = 2*(x**3 - x**2 - 2*x - 6). divide both sides by 2. Current equation is -6 = x**3 - x**2 - 2*x - 6. add 6 to both sides. Current equation is 0 = x**3 - x**2 - 2*x. factor the left hand side. Current equation is 0 = x*(x - 2)*(x + 1). \#\#\#\# -1, 0, 2.
\end{framed}
\end{enumerate}

\item \textbf{Product}: An equation that contains a product of two factors. It is solved using a sequence of \textit{mostly} scaling and shifting operations, followed by expansion and more shifting operations, then factorization. Up to two solutions.
\begin{enumerate}
\item An example problem is:
\begin{framed}
\footnotesize
\setlength{\parindent}{0em}
Solve for x: -4*(x + 4)*(x + 8) = 12\\
Output your answer after \#\#\#\#.
\end{framed}

\item The emission interval $k=\infty$ solution:
\begin{framed}
\footnotesize
\setlength{\parindent}{0em}
divide both sides by -4. expand the left hand side. add 3 to both sides. factor the left hand side. \#\#\#\# -7, -5.
\end{framed}

\item The emission interval $k=1$ solution:
\begin{framed}
\footnotesize
\setlength{\parindent}{0em}
divide both sides by -4. Current equation is (x + 4)*(x + 8) = -3. expand the left hand side. Current equation is x**2 + 12*x + 32 = -3. add 3 to both sides. Current equation is x**2 + 12*x + 35 = 0. factor the left hand side. Current equation is (x + 5)*(x + 7) = 0. \#\#\#\# -7, -5.
\end{framed}
\end{enumerate}
\end{enumerate}

We chose this task because it required somewhat diverse forms of generalization (e.g., parsing a word problem), it was not overly simple (there are many possible options at each step) but could still be created and solved algorithmically. All problems are guaranteed to have integer solutions. 

The problems had varying depths (number of operations required to solve them). Each depth between 3 and 10 inclusive contributed to roughly 10\% of the problems. Only 5.8\% of problems had depth 11 and only 1.1\% of problems had depth 12. The depth distribution is a result of how problems were sampled, using the following procedure:
\begin{enumerate}
    \item Choose a problem type out of the five options above and the answer(s) to the problem uniformly at random. Answers were all integers and sampled between [-12,12] for linear problems or linear word problems, [-9,9] for quadratic problems, [-5,5] for cubic problems, and [-8,8] for product problems. The answer ranges narrow with the number of answers to limit the size of any individual coefficient. 
    \item Write the “core” state of the problem. For quadratic and cubic tasks, this is the expanded state just before factorization. For linear it might be \texttt{x=4} or for quadratic \texttt{x**2 - 4 = 0}. For product, the task is written three operations before the final answer in its factored form, e.g., \texttt{(x+4)*(x+8) = -3}. These core states are forced to be consistent with the final answer(s) chosen in the previous step. For example, the product task requires that the sum of both answers is equal to the sum of the values in each factor (-4,-8 in the above example). Given this constraint, other values are chosen at random.
    \item Apply alternating scaling and shifting operations to both sides. They do not change the answer. The operations alternate since, e.g., two shifting operations in a row (+4-2) is directly equivalent to a single shifting operation (+2). 
    \begin{enumerate}
        \item The number of times these operations are applied to a problem varies so that problem types that have more required steps (e.g., quadratics require factorization) have more constrained operations. It is sampled uniformly at random from: [1,11] for linear problems, [1,9] for linear word problems, [1,9] for quadratic problems, [1,8] for cubic problems, and [0,8] for product problems. 
        \item The scaling operations can be “multiply” by an integer factor in {2, 3, 4,-2,-3,-4} with probability 0.8 or “divide” by {2,3,4} with probability 0.2. Division has a lower probability to avoid sampling problems with enormous (unsimplified) fractions. The specific values for a given problem are chosen uniformly at random from their allowed set.
        \item The shifting operation is addition of an integer in [-15,-1] U [1,15]. The specific value is chosen uniformly at random from this set.
    \end{enumerate}
    \item Take the equation after applying shifting and scaling operations. Apply an x-shift (add or subtract $a*x$ to both sides for integer $a \in [1,6]$) with probability 0.35 for linear, 0.0 for linear word problems, 0.3 for quadratic problems and 0.25 for cubic and product problems (more difficult problem types had a lower chance of x-shift).
    \item Apply a flip (switch the right hand side and left hand side) with probability 0.25 across all problem types except linear word problems, which this does not apply to. 
    \item If the problem type is a linear word problem, apply the word problem template.
    \item Check against existing problems in the training or test set to make sure it is not a duplicate. Filter out any duplicates. Verify the problem simplifies to the desired answer using SymPy.
\end{enumerate}

We applied this procedure to create 50,000 train and 1,000 test samples for each of the three replicate experiments. The train and test samples are verified to be distinct from one another. The ground-truth solutions for each interval $k$ were generated algorithmically by working backwards through the operations. The generated reasoning traces were checked by a general solver which applied the stated operations in order and verified if they produce the correct answer. 

\subsection{model and training}
We fine-tuned and evaluated the Qwen2.5-1.5B-Instruct model on these tasks with LoRA adapters. One LoRA adapter was trained for each condition. The LoRA adapters had rank 32, $\alpha=32$, dropout 0.05, no bias and targeted the following kinds of matrices: \texttt{q\_proj, k\_proj, v\_proj, o\_proj, gate\_proj, up\_proj, down\_proj}. This list includes the multi-layer perceptron gate, up projection and down projection weights. The adapters covered roughly 2.3\% of the parameters, and all training and evaluations were in bfloat16. 

The inputs followed the chat template with a system prompt:
\begin{framed}
\footnotesize
\setlength{\parindent}{0em}
You solve single-variable equations. Work step-by-step, one algebraic operation per sentence, and give the integer solutions at the end, comma-separated if there is more than one.
\end{framed}
The `user' field contained the problem statement and the answer was produced in the `assistant' field. The training loss was next-token cross-entropy loss on the tokens in the assistant field (the reasoning and answer tokens). 

Training was done over 50,000 training samples for 3 epochs. Each of the three replicate experiments required fine-tuning with a different training set. The optimizer was AdamW with 0 weight decay, learning rate 1e-4, default $\beta=(0.9,0.999)$. We used a cosine scheduler with 3\% warmup and a small batch size of 32. Gradients were clipped to 1.0. 

\subsection{evaluation}
We evaluated models on 1,000 held-out test samples using greedy decoding. The test samples were different for each of the three replicate experiments. Decoding proceeds one reasoning step at a time, and perturbations to the chain of thought never interrupt a step in progress. At each step, the model generates the relevant tokens that are cut off at a period (or a \#\#\#\# at the last step). A step can be followed by a state emission, which also ends in a period. The maximum number of steps (not including state emissions) in decoding was 96, and the maximum number of tokens per step (and per state emission) was 64. We trained and evaluated models up to a state emission interval of k=7. This ensured that more than 50\% of the data for any condition would have at least one state emission (53.3\% of the test data at $k=7$, to be precise).

\subsubsection{decoding with perturbations}
Random, answer-preserving steps were added to the reasoning traces with perturbation probability $p$ at any step. State emissions start with the tokens “Current equation is” and are not changed by perturbations. With probability $p$, the current operation is replaced by a multiply, divide, add or subtract operation sampled uniformly at random. The allowed multiply or divide integers are in {2,3,4,5} and the allowed add or subtract integers are in the range [1,12]. This replacement step is included in the context before the model emits the next state or outputs the next step. 

For example, for emission interval $k=1$:
\begin{framed}
\footnotesize
\setlength{\parindent}{0em}
add 12 to both sides.\\
Current equation is 64 = -12 * x - 44.\\
divide both sides by 5.         ← PERTURBATION\\
Current equation is 64/5 = -12 * x/5 - 44/5   ← The model must predict this.
\end{framed}
Recall that operations are applied to both sides of the equation, so a perturbation will never change the final answer. A robust reasoning model should overcome this detour and output the correct answer. 

\subsubsection{measuring accuracy}
The final-answer accuracy was measured by finding the first occurrence of `\#\#\#\#' and evaluating the following string. That string was separated based on commas or whitespace, converted to integer values, and compared with the correct answer(s). 

The derivation accuracy measured the correctness of the final answer as well as the correctness of the intermediate reasoning operations (not the state emissions). Computing this metric first required generating a list of reasoning operations and state emissions from the output string, which were separated by periods. We threw out any state emissions since they are not evaluated, i.e., any element of the list that starts with \texttt{Current equation is}. The sequence of reasoning steps, e.g., [\texttt{add 12 to both sides.}, \texttt{divide both sides by 2.}], was parsed algorithmically and applied to the problem statement. If the applied steps resulted in the correct simplified form, e.g., \texttt{x=10} or \texttt{(x-5)*(x-3)=0}, and the final answer was correct, then the derivation was also considered correct. 

The ignored accuracy measured the derivation accuracy while removing any steps that were perturbed. The perturbation positions were always tracked during decoding, so this evaluation used the same evaluation mechanism as derivation accuracy and ignored the steps at the perturbation positions.

The replaced accuracy tested if the model was applying the correct steps in order by implicitly replacing any externally added incorrect steps. It measured the derivation accuracy if the perturbation steps were replaced by the correct steps at those positions. Each problem had a canonical ordering of steps to solve it, so it was possible to take the perturbed steps and replace them with the canonical step that would have been in its position. If the model output was longer than the canonical answer, it was considered incorrect. 

We also evaluated intermediate state emissions using two different definitions of correctness. The first, called anchored, is whether or not a state emission is correct given the previous state emissions and the operations leading up to the current state emission. The second, called cumulative, is whether or not a state emission is correct given the starting equation and the entire history of operations, while ignoring any prior state emissions. In both cases, the correct state emission was computed algorithmically. Then, we used SymPy to check whether the correct state matched the output of the model. If the cumulative accuracy is higher than anchored accuracy, that is an indication that the model tends to recover from mistakes in its state emissions in future states. If the opposite is true, it indicates that the model drifts due to an incorrect state-emission. 

\subsection{testing alternative hypotheses for the discrepancy between derivation accuracy and final-answer accuracy}
\begin{figure}
    \centering
    \includegraphics[width=0.4\linewidth]{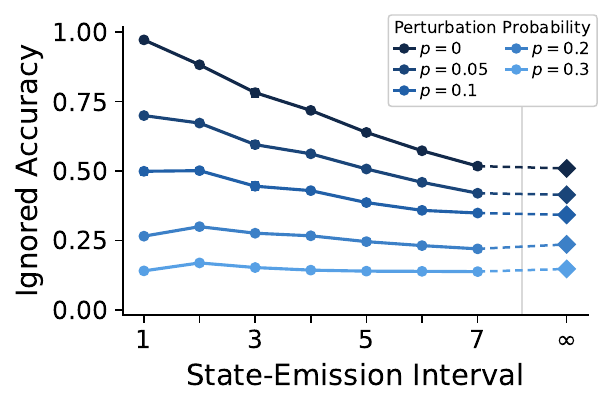}
    \includegraphics[width=0.4\linewidth]{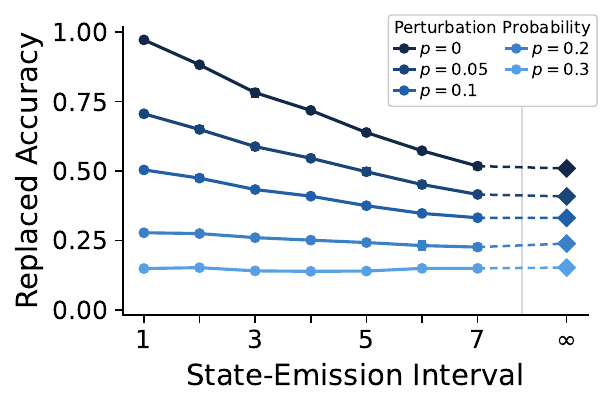}
    \caption{The ignored and replaced accuracies show similar trends as the derivation accuracy in \cref{fig:math_eval_trends} (left). They do not explain the discrepancy between the final-answer accuracy in \cref{fig:math_eval_acc} and the derivation accuracy in \cref{fig:math_eval_trends} (left).}
    \label{fig:ignored_replaced}
\end{figure}

We explored two alternative hypotheses about why training models to use large state-emission intervals $k$ leads them to be more robust to perturbations. Our results indicate this is because large $k$ causes models to fall back on pattern recognition by default. However, it is possible that these models used step-by-step reasoning but either (i) ignored operations that were added due to perturbations and reasoned around them, or (ii) implicitly replaced the perturbed operations with a correct one. 

To test both of these, we computed the ``ignored accuracy'' and the ``replaced accuracy''. They are computed similarly to the derivation accuracy except, the ignored accuracy skips any steps that were added due to perturbations, and the replaced accuracy implicitly replaces those steps with the canonically correct one at that position. \cref{fig:ignored_replaced} shows that neither of these measures recovers the trend seen in the final-answer accuracy. 

\section{details for evaluating qwen3 models on reasoning benchmarks in \cref{fig:knockout_exp}}
\label{app:knockout}
To generate answers using Qwen3-14B and Qwen3-32B (\cref{fig:knockout_exp}), we used top-$p$ decoding with $p=0.95$ and a temperature of $0.6$. Thinking mode was on. The max context length was 12,800 tokens. We evaluated exactly 192 randomly sampled problems per benchmark.

\end{document}